%% file: collas2022_conference.tex
\documentclass{article} 
\usepackage{collas2022_conference,times}

\input{math_commands.tex}

\usepackage{hyperref}
\hypersetup{
    colorlinks=true,
    linkcolor=red,
    filecolor=magenta,
    urlcolor=blue,
    citecolor=purple,
    pdftitle={Overleaf Example},
    pdfpagemode=FullScreen,
    }

\usepackage{url}

\usepackage{multirow}
\usepackage{graphicx}
\usepackage{algorithm}
\usepackage[noend]{algpseudocode}
\usepackage{subcaption}
\usepackage{amsmath}
\usepackage{amsthm}
\usepackage{xcolor}
\usepackage{pgfplots}
\usepackage{tikz}
\pgfplotsset{compat=1.16}
\usetikzlibrary{positioning,patterns,patterns.meta,calc,arrows.meta}
\usepackage{enumitem}
\usepackage[normalem]{ulem}               

\usepackage{wrapfig}
\usepackage{microtype}

\newcommand{\Reals}{\mathbb{R}}

\newcommand{\Revision}[1]{{#1}}
\newcommand\redout{\bgroup\markoverwith
{\textcolor{red}{\rule[0.5ex]{2pt}{0.8pt}}}\ULon}

\newtheorem{definition}{Definition}
\newtheorem{proposition}{Proposition}

\newcommand{\model}{\mathcal{M}}
\newcommand{\weights}{\mathbf{W}}
\newcommand{\features}{\bm{f}}
\newcommand{\feature}{\bm{f}}
\newcommand{\inputs}{\bm{x}}
\newcommand{\scores}{\bm{s}}
\newcommand{\thresholds}{\bm{th}}

\title{SHELS: Exclusive Feature Sets for Novelty Detection and Continual Learning Without Class Boundaries}

\author{Meghna Gummadi{\normalfont,} David Kent{\normalfont,} Jorge A. Mendez{\normalfont, and} Eric Eaton\\
Department of Computer and Information Science\\
University of Pennsylvania\\
\texttt{\{meghnag,dekent,mendezme,eeaton\}@seas.upenn.edu} \\
}

\collasfinalcopy 

\begin{document}

\maketitle

\begin{abstract}

While deep neural networks (DNNs) have achieved impressive classification performance in closed-world learning scenarios, they typically fail to generalize to unseen categories 
in dynamic open-world environments, in which the number of concepts is unbounded.  In contrast, human and animal learners have the ability to incrementally update their knowledge by recognizing and adapting to novel observations. In particular, humans characterize concepts via exclusive (unique) {\em sets} of essential features, which are used for both recognizing known classes and identifying novelty. Inspired by natural learners, we introduce a Sparse High-level-Exclusive, Low-level-Shared feature representation (SHELS) that simultaneously encourages learning exclusive sets of high-level features and essential, shared low-level features. The exclusivity of the high-level features enables the DNN to automatically detect out-of-distribution (OOD) data, while the efficient use of capacity via sparse low-level features permits accommodating new knowledge. 
The resulting approach uses OOD detection to perform class-incremental continual learning without known class boundaries. We show that using SHELS for novelty detection results in statistically significant improvements over state-of-the-art OOD detection approaches over a variety of benchmark datasets. 
Further, we demonstrate that the SHELS model mitigates catastrophic forgetting in a class-incremental learning setting, enabling a combined novelty detection and accommodation framework that supports learning in open-world settings.
\end{abstract}

\section{Introduction}

The successes of deep neural networks (DNNs) have occurred mainly in closed-world 
learning paradigms, where models are trained to handle fixed distributions of learning problems. To operate in realistic open-world environments, DNNs must be extended to support learning after deployment to adapt to new information and requirements
~\citep{liu2020learning}. This necessitates both the abilities to \textit{detect} novelty, as addressed by out-of-distribution (OOD) detection, 
and \textit{accommodate} this novelty into the learnt model, as addressed by continual or lifelong learning. However, these two problems have rarely been examined together. 
The work we present here bridges these two fields by developing a novel 
OOD detector that automatically triggers accommodation of new knowledge in a natural continual learning loop.

Work in continual learning 
focuses on learning tasks (or classes) consecutively, incrementally acquiring new behaviors over changing data distributions~\Citep{kirkpatrick2017overcoming, jung2020continual,yang2017deep, lee2019learning, lee2021sharing,shin2017continual, van2020brain}. 
However, most existing approaches must be explicitly told when new tasks are introduced.  In contrast, humans automatically detect when current percepts do not match known patterns~\citep{weizmann1971novelty}, 
spurring the refinement or introduction of new knowledge without explicit task or class boundaries~\citep{colombo1983infant, hunter1983effects}. In this way, novelty detection and accommodation are inherently intertwined in natural learners, with the learnt representations both characterizing known concepts and identifying their boundaries. \citet{tversky1977features} describes this process as a feature matching problem, whereby humans represent objects as a collection of essential features, enabling them to recognize novel instances via 1) the presence of a new feature, 2) the absence of known features, or 3) a new combination of known and unknown features.
Motivated by this idea, we explore how we can incorporate such representations into open-world continual learning.

The key insight of our work is to represent each class as an \emph{exclusive set} of high-level features within a DNN, while lower level features can be shared. 
The SHELS formulation realizes the exclusivity at the high levels by leveraging cosine normalization, and encourages sharing at the lower levels via group sparsity, leading to representations that consist of only essential features. Such structures identify combinations of features that encapsulate a unique {\em signature} for each class. Comparing the features triggered by OOD percepts with known class signatures enables the learner to detect novel classes and accommodate them, endowing it with the ability to learn continually without class boundaries.

The key contributions of this work include:
\begin{itemize}[itemsep=4pt,topsep=4pt,parsep=0pt,partopsep=0pt]
\vspace{-.5em}
\item We describe a novel mechanism for detecting OOD data within DNNs by representing concepts as exclusive combinations of reusable lower level sparse features. We call this representation SHELS.
\item We develop a continual learning algorithm that leverages these exclusive feature sets in tandem with sparsity to incrementally learn concepts, mitigating catastrophic forgetting of previously learnt knowledge.
\item We demonstrate a combined framework that supports class-incremental continual learning without the specification of explicit class boundaries, enabling the agent to operate effectively in an open-world setting through the integration of novelty detection and accommodation.
\end{itemize}

\section{Representing concepts as exclusive sets of higher level features}
\label{sec:ExclusiveFeatureSets}

It is well-established in  developmental psychology that children characterize concepts as collections of essential features, known as schemas~\citep{Piaget1926Language}. As an example, a child living with a dalmatian dog may learn to characterize it by its four-legs, fur, tail, black and white spots, and medium size. These features constitute the child's {\sf\small dog} schema. Critically, schemas can share features (enabling transfer between concepts), but are differentiated by their unique {\em sets} of essential features~\citep{tversky1977features}.  The child may also have a {\sf\small horse} schema that shares the four-legs, fur, and tail features, but is larger. 
When the child first encounters a cow, it may differentiate it as a new animal, since it has never before seen a large animal with four-legs, spots, and a tail. It is the exclusive {\em set} of essential features that enables the child to both recognize concepts and detect  that the cow is OOD. The child may mistakenly call the cow a horse, using their closest matching schema. With their parent's correction, the child can create a new {\sf\small cow} schema to accommodate the new information. Thus, both features and schemas are continually acquired  and refined over time.

From a mathematical standpoint, schemas represent a collection of  exclusive sets of features. Precisely, we can state:
\begin{definition}[Exclusive Sets]
Two sets $\mathcal{S}$, $\mathcal{T}$ are exclusive if and only if they each contain unique elements: $\mathsf{exclusive}(\mathcal{S},\mathcal{T}) \iff \mathcal{S} - \mathcal{T} \neq \emptyset \land \mathcal{T} - \mathcal{S} \neq \emptyset$.
{\normalfont (Note that this definition inherently precludes subset relationships and prevents $\mathcal{S}$ and $\mathcal{T}$ from being empty.)}
\end{definition}
\begin{definition}[Collection of Exclusive Sets]
A collection of sets $\mathcal{C}$ is exclusive if and only if every pair of sets within that collection is exclusive: $\mathsf{exclusive}(\mathcal{C}) \iff \forall \mathcal{S}, \mathcal{T}\in \mathcal{C} \ \ \mathcal{S}\neq\mathcal{T} \implies \mathsf{exclusive}(\mathcal{S},\mathcal{T})$.
\end{definition}

To manifest the idea of schemas in a representational space for machine learning, we embed each data instance $\bm{x} \in \Reals^d$ as a set of derived features $\bm{f} = g(\bm{x}) \in \Reals^k$. Our goal is for the non-zero entries of the embeddings  $\bm{f}$ of each class to form a collection of exclusive sets. {\em Orthogonality} between the vectors in the resulting vector space captures this notion of exclusive feature sets. In particular, a set of orthogonal feature embeddings $\bm{f}_1, \ldots, \bm{f}_C \in \Reals^k$ for each of $C$ classes forms a collection of exclusive sets (see Appendix~\ref{sec:ProofOfPropOrthogonalityCapturesExclusivity} for the proof).
In fact, orthogonality is more stringent than exclusivity, since non-orthogonal features can still constitute exclusive sets. However, our work shows empirically that orthogonal feature embeddings are sufficient for both characterizing classes and enabling OOD detection.

To impose orthogonality, we employ cosine normalization~\citep{luo2018cosine} and use cosine scores in the loss function when learning. 
This embeds the in-distribution (ID) classes along the surface of the unit ball, spacing them apart. Intuitively, this maximizes the chance that a new OOD class would have an embedding $g(\bm{x})$ 
with low cosine similarity with the other classes, facilitating OOD detection.
A continual learner utilizing this technique can then adapt its embedding to preserve orthogonality when accommodating the new class, permitting future OOD detection of other new classes.  Figure~\ref{fig:ExclusiveFeatureSets} illustrates this process, which we develop into a complete algorithm in Section~\ref{sec:Framework}.

\begin{figure}[b!]
    \centering
    \vspace{-0.75em}
    \includegraphics[width=0.74\textwidth,clip,trim=0.2in 1.6in 0.2in 1.58in]{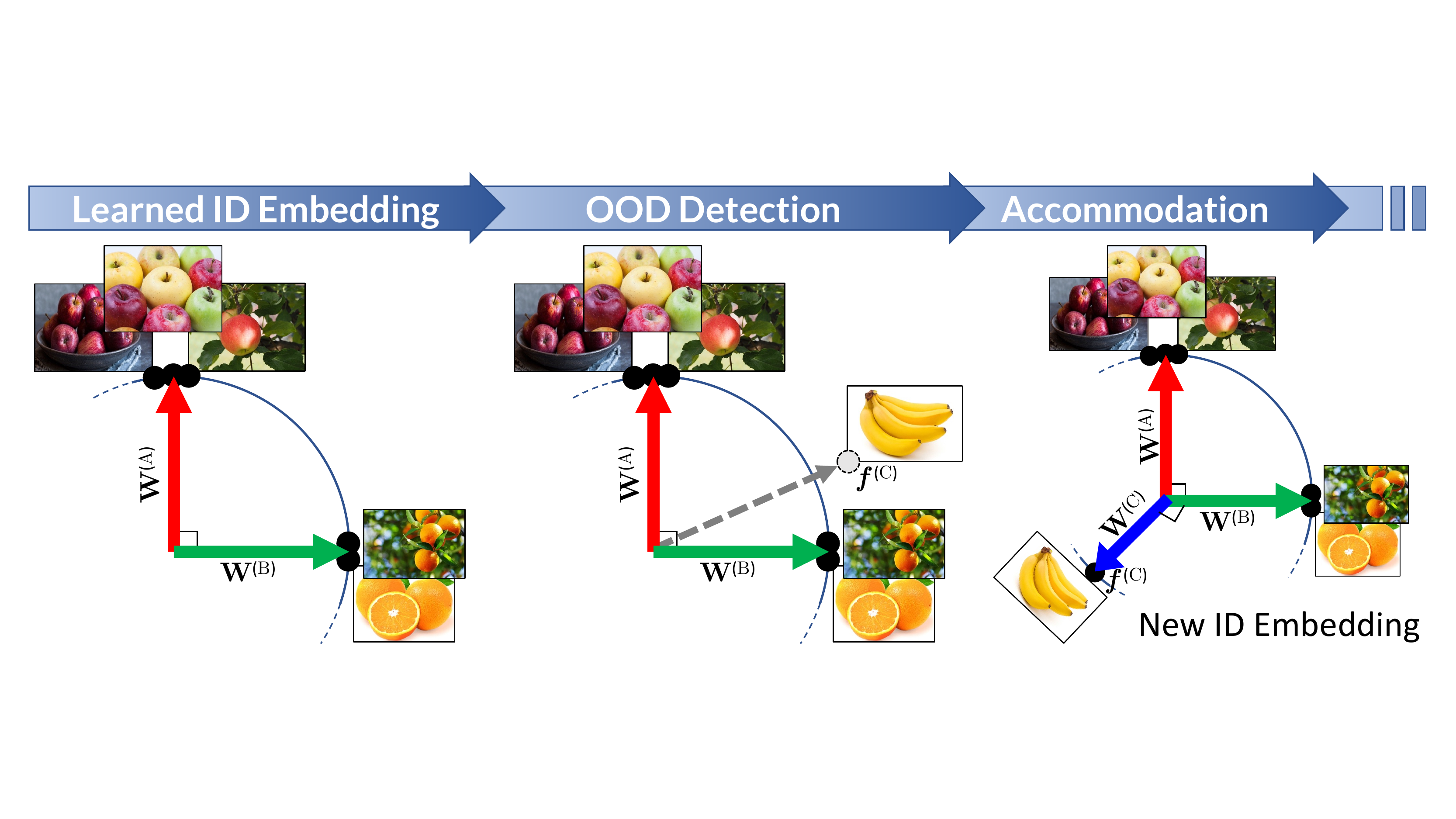}
    \caption{ID data is represented by exclusive sets of features, corresponding to orthogonal embeddings along the unit ball with broad separation of the classes. An OOD instance in the same embedding lies in a different region, enabling detection of the novelty and triggering accommodation of the embedding. \vspace{-0.5cm}}
    \label{fig:ExclusiveFeatureSets}
\end{figure}

\section{Related work}
\label{sec:relatedWork}
The notions of orthogonality and sparsity discussed above connect to work in novelty detection and continual learning. We explore those connections in the following paragraphs, and survey other (less) related literature in Appendix~\ref{sec:additionalRelatedWork}.

\textbf{Novelty Detection~~} Novelty detection has been studied in various forms, including  open-world recognition~\citep{bendale2015towards}, uncertainty estimation~\citep{lakshminarayanan2017simple}, and anomaly detection~\citep{andrews2016transfer}. A popular baseline proposed by \citet{hendrycks2016baseline} thresholds the softmax scores output by a network to detect OOD inputs. Similarly, OpenMax~\citep{bendale2016towards} compares the distance of the penultimate layer activation patterns of a sample to those of ID data.
Such distance-based methods fall short when ID and OOD data are similar, which is a challenging case that we explicitly consider for SHELS. ODIN~\citep{liang2017enhancing} complements thresholding with input pre-processing and temperature scaling to compute more-effective softmax scores. The method by \citet{lee2018simple} models every activation layer learnt using ID data as a class-conditional Gaussian distribution, and defines confidence scores using the Mahalanobis distance with respect to the nearest class-conditional distribution. However, both of these approaches assume access to sample OOD data to tune their hyperparameters, which is unrealistic in open-world settings. \citet{techapanurak2020hyperparameter} propose to compute the softmax scores over the temperature-scaled cosine similarity, outperforming both ODIN and Mahalanobis while making use of only ID data for hyperparameter tuning. Our SHELS approach also uses cosine similarity, but unlike existing formulations, SHELS encourages the embeddings of ID classes to be orthogonal (and therefore exclusive), which we demonstrate improves OOD detection.

\textbf{Continual Learning~~}
The major challenge in continual learning is avoiding catastrophic forgetting, whereby learning new knowledge interferes with previously known concepts~\citep{mccloskey1989catastrophic}. A variety of approaches have been proposed to mitigate this problem, from replay-based methods that use samples of previous task data to retain performance~\citep{lopez2017gradient,chaudhry2018efficient}, to regularization-based strategies that selectively update weights to avoid forgetting~\citep{kirkpatrick2017overcoming,zenke2017continual}. The regularization-based method by \citet{jung2020continual} uses a combination of sparsity and weight update penalties based on the importance of nodes in the network. We take a similar approach, but use a more structured form of sparsity along with high-level exclusivity to enable both continual learning and OOD detection. Several recent continual learning methods use orthogonality to avoid interference between tasks, embedding each task model on an orthogonal subspace~\citep{chaudhry2020continual,saha2021gradient,deng2021flattening, duncker2020organizing}. We explore a complementary use of orthogonality, to create exclusive feature sets that enable OOD detection during continual learning. Notably, the only prior work studying both OOD detection and continual learning reuses previously learnt representations to accommodate new classes~\citep{lee2018simple}, limiting its applicability to settings where the ID representations are sufficient to handle OOD data. \Revision{Moreover, they do not demonstrate the use of novelty detection and accommodation as an integrated system as we do.}

\vspace{-0.25cm}
\section{Framework}
\label{sec:Framework}
\vspace{-0.25cm}

Our approach trains DNNs that have 1) high-level feature sets that are exclusive to each class and 2) low-level features that are shared among classes. We call this the {\em SHELS} representation---a Sparse High-level-Exclusive  Low-level-Shared representation.  SHELS representations promote effective class discrimination and form capacity-efficient models, enabling both novelty detection and accommodation. 
Learning features this way differs substantially from typical DNN learning, which does not optimize for network efficiency and therefore exploits the full representational capacity, leading to representations that are redundant, correlated, and sensitive to the input distribution.  
Our approach learns SHELS representations using a combination of sparsity regularization and cosine normalization, enabling DNNs to learn features that are non-correlated, essential, and (at higher levels) exclusive to each class. 
We exploit this representation to detect new classes as  novel activation patterns of high-level features, enabling OOD detection. We then  accommodate the necessary information for new classes into the DNN by identifying and updating only those nodes that are unimportant to previously learnt classes, thereby mitigating catastrophic forgetting for continual learning.

\subsection{Exclusive high-level feature sets with shared low-level features}

We next develop the regularization terms and overall loss function of our approach.

\textbf{Notation~~} For a neural network of $L$ hidden layers, let $\weights^{l}$ denote the weight matrix for layer $l \in {1,\ldots,L}$, with all weight matrices collected together into a tensor of parameters $\weights$.
The vector of outgoing weights from node $n_g$ in layer $l$, which we denote a \textit{group}, is given by $\weights_{[g, :]}^{l}$, while the vector of incoming weights to node $n_i$ is given by $\weights_{[:, i]}^{l}$. $\Omega (\weights^{l})$ is a regularization term on the network weights at layer $l$ and $\rho(n)$ is the average activation of a node $n$ over a dataset. The features produced by layer $l$ of the network on input $\inputs$ are represented by $\features^{l}(\inputs)$ and $\feature_{n}(\inputs)$ denotes the feature activation at node $n$. We use $c \in {1,\ldots,C}$ to denote seen ID classes and $D_{1:C}$ to denote the dataset over all seen classes. In continual learning formulations, we use $\weights^{(c)}$ to denote the weight tensor trained up to the $c$-th class.

\subsubsection{Exclusive high-level features}
\label{sec:exlusive-high}

\definecolor{color1}{rgb}{0.00392157, 0.45098039, 0.69803922} 
\definecolor{color2}{rgb}{0.8254902 , 0.78235294, 0.2       } 
\definecolor{color3}{rgb}{0.83529412, 0.36862745, 0.   } 
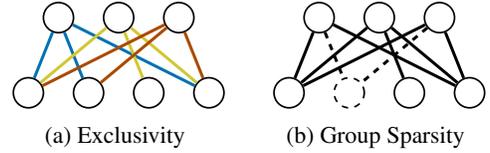
\begin{wrapfigure}{r}{0.4\textwidth}
\begin{minipage}[t]{0.4\textwidth}
    \vspace{-2.75em}%
    \centering
    \subfloat[Exclusivity\label{fig:exclusivity}] {
        \scalebox{0.8}{
        \begin{tikzpicture}[main/.style={draw, thick, circle, minimum size=0.5cm, scale=1.0}]
            \node[main] at (0.5,1.25) (S1) {};
            \node[main] at (1.5,1.25) (S2) {};
            \node[main] at (2.5,1.25) (S3) {};
            \node[main] at (0,0) (D1) {};
            \node[main] at (1,0) (D2) {};
            \node[main] at (2,0) (D3) {};
            \node[main] at (3,0) (D4) {};
            \draw[line width=0.5mm,color1] (S1) -- (D1);
            \draw[line width=0.5mm,color1] (S1) -- (D2);
            \draw[line width=0.5mm,color1] (S1) -- (D4);
            \draw[line width=0.5mm,color2] (S2) -- (D1);
            \draw[line width=0.5mm,color2] (S2) -- (D3);
            \draw[line width=0.5mm,color2] (S2) -- (D4);
            \draw[line width=0.5mm,color3!80!black] (S3) -- (D1);
            \draw[line width=0.5mm,color3!80!black] (S3) -- (D2);
            \draw[line width=0.5mm,color3!80!black] (S3) -- (D4);
        \end{tikzpicture}
        }
    }
    \hfill
    \subfloat[Group Sparsity\label{fig:groupSparsity}] {
        \scalebox{0.8}{
        \begin{tikzpicture}[main/.style={draw, thick, circle, minimum size=0.5cm, scale=1.0}]
            \node[main] at (0.5,1.25) (S1) {};
            \node[main] at (1.5,1.25) (S2) {};
            \node[main] at (2.5,1.25) (S3) {};
            \node[main] at (0,0) (D1) {};
            \node[main,dashed] at (1,0) (D2) {};
            \node[main] at (2,0) (D3) {};
            \node[main] at (3,0) (D4) {};
            \draw[line width=0.5mm] (S1) -- (D1);
            \draw[line width=0.5mm,dashed] (S1) -- (D2);
            \draw[line width=0.5mm] (S1) -- (D4);
            \draw[line width=0.5mm] (S2) -- (D1);
            \draw[line width=0.5mm] (S2) -- (D3);
            \draw[line width=0.5mm] (S2) -- (D4);
            \draw[line width=0.5mm] (S3) -- (D1);
            \draw[line width=0.5mm,dashed] (S3) -- (D2);
            \draw[line width=0.5mm] (S3) -- (D4);
        \end{tikzpicture}
        }
    }
\caption{ (a) Exclusivity encourages each upper-level node to chose an exclusive set of lower level nodes (the matching colored edges). (b)~Group Sparsity eliminates neurons that are redundant or not shared by multiple upper-level nodes (the dashed-edge neuron). \vspace{-3em}}
\label{fig:Group_Excl_sparsity}
\end{minipage}
\end{wrapfigure}
 To impose exclusivity in the high-level features (Figure~\ref{fig:exclusivity}), we employ cosine normalization in the last layer of the neural network, in place of the usual dot product. Typical classification networks predict class scores for a given input $\inputs$ by applying the softmax function to the output of the last linear layer: \mbox{$\softmax(\weights^{L}\features^{L-1}(\inputs) + \bm{b}^{L})$.} We replace this linear transformation with the cosine of the angle between the weights $\weights_{[:,c]}^{L}$ for class $c$ and the features $\features^{L-1}(\inputs)$ (i.e., the cosine similarity) to compute class scores:
\begin{equation}
    \label{equ:CosineSimilarity}
    \cos^{(c)}(\weights, \inputs) = \frac{{\weights_{[:,c]}^{L}}^ {\!\!\!\!\top}\features^{L-1}(\inputs)}{\|\weights_{[:,c]}^{L}\|
     \|\features^{L-1}(\inputs)\|}\enspace.
\end{equation}
Using the raw cosine similarity scores as the class scores  encourages the weights of the last layer to be sensitive to \emph{which} features are activated instead of the degree of the feature activations. Therefore, training the network to optimize for the raw cosine scores results in orthogonal---and consequently exclusive---high-level features per class.
\Revision{We empirically demonstrate that cosine normalization leads to exclusive features sets, and compare the exclusivity in our method to two baseline methods: the standard CNNs used by most other novelty detection approaches, and the method of \citet{techapanurak2020hyperparameter}, which uses a form of cosine similarity. See Appendix~\ref{sec:CosineNormalizationInducesExclusiveSets} for this analysis.} 


As an alternative, we explored the exclusive sparsity regularizer of \citet{Zhou2010ExclusiveLF}, which encourages competition among features selected for a class. We found it to be less effective than cosine normalization (see Appendix~\ref{sec:ExclusiveSparsityFeatureSets}), since it encourages {\em individual} features to be exclusive to a class, in contrast to {\em sets} of features being exclusive to a class.

\subsubsection{Shared lower level features}

Sparse regularizers are primarily used to reduce the complexity of DNNs by zeroing out certain weights. However, we can also use sparse regularizers to 
enable continual learning.
For class-incremental continual learning, we need the ability to accommodate new information into the network without disrupting any knowledge learnt for previous classes. We accomplish this by creating sparsity in the network during training, allowing us to maintain unused nodes for learning future classes. We impose such sparsity in a structured fashion using the group Lasso regularizer defined by \citet{yuan2006model} and adapt it to DNNs~\citep{yoon2017combined, wen2016learning}, promoting the removal of redundancies by eliminating correlated features while also encouraging the sharing of features:
    \begin{equation}
        \Omega_{G} (\weights^{l}) = \sum {\left\lVert \weights_{g}^{l} \right\rVert}_{2} = \sum_{g}\sqrt{\sum_{i} \left(\weights_{[g,i]}^{l}\right)^{2}}\enspace.
    \end{equation}
Group Lasso promotes inter-group sparsity; that is, it eliminates groups (nodes) in the lower level that are redundant and are not shared among multiple high-level groups, as illustrated in Figure~\ref{fig:groupSparsity}.

\subsubsection{Layered sparse representation learning}
\label{sec:shells_loss}
To learn SHELS representations, we combine group sparsity at different layers with cosine normalization at the last layer of the network by optimizing for the following loss:
\begin{equation}
     \label{equ:Loss}
    \mathcal{L}(\weights) = \underbrace{\frac{1}{|D_{1:C}|}\sum_{c=1}^{C}\sum_{k=1}^{|D_c|} -\log\Big(\cos^{(c)}(\weights, \inputs_k)\Big)}_{\mathcal{L}_{CE}} + \ \alpha\sum_{l}\Big(1-\mu^{(l)}\Big)\biggl(\underbrace{\sum_{g}{\left\lVert \weights_{[g, :]}^{l} \right\rVert}_{2} \biggr)}_{\Omega_{G}}\enspace,
\end{equation}where $\mathcal{L}_{CE}$ is the standard cross-entropy loss applied to the cosine similarity scores of Equation~\ref{equ:CosineSimilarity} and $\Omega_{G}$ is the group sparsity regularizer. The parameter $\alpha$ controls the amount of regularization, and by setting $\mu^{(l)} = \frac{l-1}{L-1}$ we can interpolate between more sharing at the lower layers and less sharing (and therefore more exclusivity) at the higher layers, with $\mu^{(1)}=0$ and $\mu^{(L)}=1$ at the outermost layers.  \Revision{We provide an empirical analysis of the effects of group sparsity regularization on OOD detection by varying the $\alpha$ and $\mu^{(l)}$ hyperparameters in Appendix~\ref{sec:additionalOODExp}.}

\newpage

\subsection{Novelty detection}
\label{sec:novelty-detection}


\begin{wrapfigure}{r}{0.43\textwidth}
\vspace{-4.5em} 
\begin{minipage}{\linewidth}
\begin{algorithm}[H]
\footnotesize
\caption{Novelty Detection Algorithm}\label{alg:OODDetection}
\begin{algorithmic}[1]
\Procedure{initTraining} {$D_{1:C},\mu, \alpha, T$}
 \State Randomly initialize model $\model$ parameters, $\weights$
 \State Define loss $\mathcal{L}(\weights)$ as in Equation~\ref{equ:Loss}
 \For{$t = 1,\ldots, T$} \Comment{Adam optimizer}
    \State $\weights_{t} \gets \weights_{t-1} - \nabla \mathcal{L}(\weights_{t-1})$
\EndFor
\State $\thresholds \gets$ \texttt{computeTh($\model, D_{1:C}$)}
\State \textbf{Output:} $\model,$ $\thresholds$    \Comment{Trained model}
\EndProcedure

\Procedure{detect} {$X, \model$}
\State $ \scores \gets$ \texttt{computeScore($X,\model$)}
\State $ s^* \gets \max(\scores)$
\State $ c^* \gets \argmax(\scores)$
\If{$s^* > \thresholds_{c^*}$}
    \State \textbf{Output:} class $c^*$
\Else
    \State \textbf{Output:} $novel$
\EndIf
\EndProcedure
\end{algorithmic}
\end{algorithm}
\end{minipage}
\vspace{-1.0em} 
\end{wrapfigure}
With the learnt model composed of SHELS features, we exploit the exclusivity in the high-level feature sets with respect to each known class to detect novel classes. Since each class is represented by an exclusive set of high-level features (their {\em signature}), a novel class is detected when the high-level feature activations differ from the signatures of known (ID) classes. Concretely, we define the signature of class $c$ as the mean high-level features over all seen data for that class.

Our novelty detector (Algorithm~\ref{alg:OODDetection}) first trains a model $\model$ for $T$ epochs to learn SHELS features using the ID training data $D_{1:C}$ (lines 2-5). For an input $\inputs$, we compute a class similarity score $\scores_c$ for each ID class $c \in1,\ldots, C $ via:
\begin{equation}
    \label{equ: ClassSimScore}
    \scores_{c}(\inputs) = {\weights_{[:,c]}^{L}}^ {\!\!\!\!\top}\features^{L-1}(\inputs)\enspace,
\end{equation}
which measures  similarity between the high-level features of the sample and the signature of known ID class $c$.
We then compute thresholds $\thresholds$ for each ID class $c$ as the mean of $\scores_c$ minus one standard deviation, computed over the training data of each class that the model correctly classifies 
(line 6). At inference time (lines 8-15), we compute the similarity score $\scores(\inputs)$ with respect to each class $c$ and find the maximum over the classes' scores $s^*(\inputs)=\max\scores(\inputs)$ and the corresponding class $c^*=\argmax\scores(\inputs)$. Then, $\inputs$ is predicted as belonging to class $c^*$ if its similarity score $s^*(\inputs)$ is greater than the threshold for the class, $\thresholds_{c^*}$; otherwise, it is identified as \textit{novel}.
\Revision{Note that the method of \citet{techapanurak2020hyperparameter} uses  normalized cosine similarities when computing the class scores which can amplify arbitrarily small feature activations, resulting in spurious detections. Our approach of computing class scores using Equation~\ref{equ: ClassSimScore} addresses this shortcoming, resulting in better novelty detection performance as shown by our results in Section~\ref{sec:detect-exp}}.



\subsection{Novelty accommodation}
\label{sec:novelty-accomm}

\newcommand{\circlesize}{0.4cm}
\newcommand{\childspace}{0.25}
\newcommand{\parentspace}{0.75}
\newcommand{\diagspace}{0.5} 
\newcommand{\legendspace}{0.08}
\newcommand{\legendtext}{\scriptsize}
\begin{wrapfigure}{r}{0.52\textwidth}
\begin{minipage}[t]{0.52\textwidth}
    \centering
    \vspace{-3.em}
    \scalebox{1.0}{
    \begin{tikzpicture}[main/.style={draw, thick, circle, minimum size=\circlesize}]
        \node[main,fill=color1] at (0,0) (C1) {};
        \node[main,fill=color1] (C2) [right=\childspace of C1] {};
        \node[main] (C3) [right=\childspace of C2] {};
        \node[main,fill=color1] (C4) [right=\childspace of C3] {};

        \draw[white] (C1) -- (C2) node[midway] (C12) {};
        \node[main] (B1) [above=\parentspace of C12] {};
        \node[main,fill=color1] (B2) [right=\childspace of B1] {};
        \node[main,fill=color1] (B3) [right=\childspace of B2] {};
        \path[thick] (B1) edge (C1) edge (C2) edge (C3) edge (C4);
        \path[thick] (B2) edge (C1) edge (C2) edge (C3) edge (C4);
        \path[thick] (B3) edge (C1) edge (C2) edge (C3) edge (C4);

        \draw[white] (B1) -- (B2) node[midway] (B12) {};
        \node[main] (A1) [above=\parentspace of B12] {};
        \node[main,fill=color1] (A2) [right=\childspace of A1] {};
        \path[thick] (A1) edge (B1) edge (B2) edge (B3);
        \path[thick] (A2) edge (B1) edge (B2) edge (B3);
        \draw[white] (A1) -- (A2) node[midway] (d1) {};
        \node (d1t) [above=0.15 of d1] {after learning $C$};

        \node[main,fill=color1] (C12) [right=\diagspace of C4] {};
        \node[main,fill=color1] (C22) [right=\childspace of C12] {};
        \node[main] (C32) [right=\childspace of C22] {};
        \node[main,fill=color1] (C42) [right=\childspace of C32] {};

        \draw[white] (C12) -- (C22) node[midway] (C122) {};
        \node[main] (B12) [above=\parentspace of C122] {};
        \node[main,fill=color1] (B22) [right=\childspace of B12] {};
        \node[main,fill=color1] (B32) [right=\childspace of B22] {};
        \path[line width=0.5mm,densely dashed] (B12) edge (C12) edge (C22) edge (C32) edge (C42);
        \path[line width=0.5mm] (B22) edge[color1] (C12) edge[color1] (C22) edge[color3] (C32) edge[color1] (C42);
        \path[line width=0.5mm] (B32) edge[color1] (C12) edge[color1] (C22) edge[color3] (C32) edge[color1] (C42);

        \draw[white] (B12) -- (B22) node[midway] (B122) {};
        \node[main] (A12) [above=\parentspace of B122] {};
        \node[main,fill=color1] (A22) [right=\childspace of A12] {};
        \path[line width=0.5mm,densely dashed] (A12) edge (B12) edge (B22) edge (B32);
        \path[line width=0.5mm] (A22) edge[color3] (B12) edge[color1] (B22) edge[color1] (B32);
        \draw[white] (A12) -- (A22) node[midway] (d12) {};
        \node (d12t) [above=0.15 of d12] {learning $C+1$};

        \node[main,white,minimum size=0.15cm] (L3) [right=0.60 of B32] {}; 
        \draw[line width=0.5mm,color1] (L3.east) -- (L3.west);
        \node[main,minimum size=0.15cm] (L2) [above=\legendspace of L3] {};
        \node[main,fill=color1,minimum size=0.15cm] (L1) [above=\legendspace of L2] {};
        \node[main,white,minimum size=0.15cm] (L4) [below=\legendspace of L3] {};
        \draw[line width=0.5mm,color3] (L4.east) -- (L4.west);
        \node[main,white,minimum size=0.15cm] (L5) [below=\legendspace of L4] {};
        \draw[line width=0.5mm,densely dashed] (L5.west) -- (L5.east);

        \node (L1d) [right=0.15 of L1] {\legendtext important nodes};
        \node (L2d) [right=0.15 of L2] {\legendtext unimportant nodes};
        \node (L3d) [right=0.15 of L3] {\legendtext frozen $W$};
        \node (L4d) [right=0.15 of L4] {\legendtext frozen with $W=0$};
        \node (L5d) [right=0.15 of L5] {\legendtext learnable weights};
    \end{tikzpicture}}
    \caption{Novelty accommodation via selective weight updates. Important nodes (blue) for classes $1\ldots C$ are identified by $\rho^{(C)}(n) > 0$. Incoming weights to important nodes (blue/orange lines) are frozen (penalized from updates by $\Omega_{WP}$) when learning class $C+1$. To avoid negative transfer, orange weights from unimportant nodes are frozen to 0.}
    \label{fig:weight_freezing}
    \end{minipage}
    \vspace{-0em}
\end{wrapfigure}
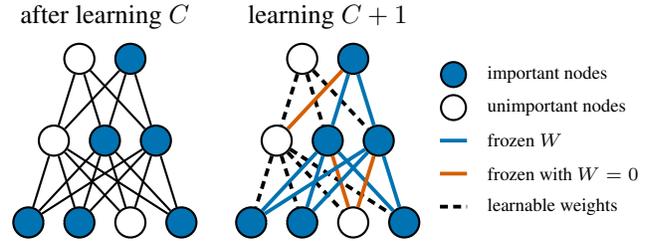
Following the detection of a novel class, our framework updates the SHELS representation to accommodate it, which involves learning new knowledge relevant to the novel class while preserving previously learnt knowledge to prevent catastrophic forgetting~\citep{mccloskey1989catastrophic}.
Since group sparsity has the effect of eliminating certain nodes from the network, 
these unimportant nodes remain available for 
incorporating new knowledge.
We measure a node's importance as the average activation of the node over the entire training data; 
\citet{jung2020continual} use a similar method to compute node importance. Precisely, we measure the importance of a node $n$ up to class $C$ as:
\begin{equation}
     \label{equ:node_imp}
        \rho^{(C)}(n) =  \frac{1}{|D_{1:C}|}\sum_{k=1}^{|D_{1:C}|}\feature_{n}(\inputs_{k})\enspace.
        \vspace{-0.5em}
\end{equation}

Our approach focuses the learning of the new class on unimportant nodes, while ensuring that important nodes are not modified. 
In particular, we target two types of modifications that could affect important nodes:
\textit{model drift} and \textit{negative knowledge transfer}.
Model drift (illustrated in Appendix~\ref{sec:SupplementalIllustrations},  Figure~\ref{fig:catastrophic_forgetting:model_drift}) occurs when the incoming weights of important nodes are changed during the learning of the new class. To avoid model drift, we penalize updates to these incoming weights. On the other hand, negative knowledge transfer occurs when an unimportant node is updated and used as input to an important node (see Appendix~\ref{sec:SupplementalIllustrations}, Figure~\ref{fig:catastrophic_forgetting:negative_knowledge_transfer}). To avoid negative knowledge transfer,
we set the weights connecting unimportant nodes (i.e., $\rho^{(C)}(n)=0$) to important nodes to zero and restrict their updates, thereby stabilizing important nodes. The combined approach for avoiding forgetting is illustrated in Figure~\ref{fig:weight_freezing}.

To selectively update weights, we add a weight penalty regularizer $\Omega_{WP}$ to Equation~\ref{equ:Loss} to penalize the incoming weights $\weights_{[:, i]}$ of node $n_i$ from being updated, weighted by the node's importance $\rho^{(C)}(n_i)$. This gives us the loss:
\begin{equation}
     \label{equ:Contll_Loss}
    \mathcal{L}(\weights^{(C+1)}) = \mathcal{L}_{CE} + \alpha\sum_{l}\sum_{g}\biggl((1-\mu){\left\lVert \weights_{[g, :]}^{(C+1),l} \right\rVert}_{2}\biggr)\enspace
    + \beta\sum_{l}\sum_{i}\underbrace{{\rho^{(C)}(n_{i})\left\lVert \weights_{[:, i]}^{(C+1),l} -  \weights_{[:, i]}^{(C),l} \right\rVert}_{2}}_{\Omega_{WP}}\enspace.
    \vspace{-1em}
\end{equation}

\begin{wrapfigure}{r}{0.43\textwidth}
\vspace{-1.0em}
\begin{minipage}{\linewidth}
\begin{algorithm}[H]
\footnotesize
\caption{Novelty Accommodation Algorithm}\label{alg:ContinualLearning}
\begin{algorithmic}[1]
\State $\rho^{(C)}(n) \gets $ \texttt{computeNodeImp($D_{1:C}, \model$)}
\Procedure{accomm}{$D_{C+1},\model, \alpha, \beta, \mu, T, \rho^{(C)}(n)$}
\State Define loss $\mathcal{L}(\weights^{(C+1)},\weights^{(C)})$ as in Equation~\ref{equ:Contll_Loss}
\For{$t=1,\ldots,T$} \Comment{Adam optimizer}
    \State $\weights_{t}^{(C+1)}\!\!\gets\!\! \weights_{t-1}^{(C+1)}-\nabla\mathcal{L}(\weights_{t-1}^{(C+1)}\!, \weights^{(C)})$
\EndFor
\State \textbf{Output:} $\model$ \Comment{Updated model}
\EndProcedure
\end{algorithmic}
\end{algorithm}
\end{minipage}
\vspace{-2em}
\end{wrapfigure}
This ensures that the new class is primarily learnt in unimportant nodes, without modifying important nodes. The group sparsity and cosine normalization terms further ensure the ability to detect and accommodate future novel classes.

Our novelty accommodation method (Algorithm~\ref{alg:ContinualLearning}) first computes the node importance $\rho^{(C)}(n)$ for ID classes of the SHELS model $\model$ trained on dataset $D_{1:C}$ (line 1). Then, \texttt{accomm()} uses  $\rho^{(C)}(n)$ to accommodate the novel class by updating $\model$ over $T$ epochs, optimizing Equation~\ref{equ:Contll_Loss} on data $D_{C+1}$ (lines 3-6).

\subsection{Novelty detection and accommodation for continual learning}
\label{sec:detect_accomm}

We combine the novelty detection from Section~\ref{sec:novelty-detection} with the novelty accommodation from Section~\ref{sec:novelty-accomm}, resulting in a joint framework that supports class-incremental novelty accommodation without the specification of class boundaries by performing novelty detection. Figure~\ref{fig:framework_arch} describes the framework and its components.

The continual learning agent first learns the SHELS model $\model$ over the ID data $D_{1:C}$, and computes the thresholds $\thresholds$ and the node importance values $\rho(n)$. The model is then deployed to a detection phase, using the approach from Section~\ref{sec:novelty-detection} to perform inference on unseen data $D_{\text{test}}$ in a dynamic environment, which may contain ID or OOD data. If it detects $D_{\text{test}}$ as OOD with a confidence greater than a preset threshold, the agent switches to novelty accommodation. During accommodation, the agent receives additional data of the newly detected class $D_{C+1}$ and accommodates the novel class by updating previously-unimportant nodes using the approach from Section~\ref{sec:novelty-accomm}. Critically, the agent expands the representation learnt for the previous classes $1,\ldots,C$ to incorporate the new class $C+1$ while maintaining orthogonality, ensuring that future classes can still be detected as OOD, as illustrated in Figure~\ref{fig:ExclusiveFeatureSets}. Once the new class has been accommodated, the agent switches back to deployment mode, and the process repeats.

\begin{figure}[h!]
\vspace{-.5em}
\centering
   \includegraphics[width=.9\textwidth,clip,trim=1.3in 2.75in 1.3in 2.75in]{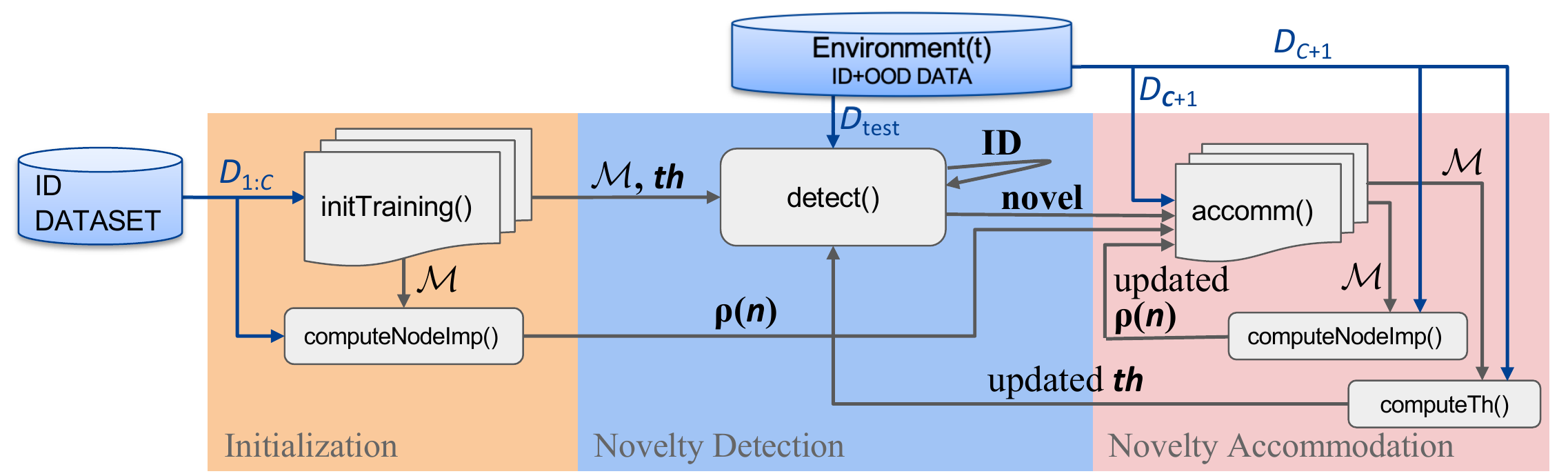}
\caption{Our framework for novelty detection and accommodation without class boundaries.}
\label{fig:framework_arch}
\end{figure}

\section{Experiments and results}

This section evaluates the effectiveness of our framework  separately for both novelty detection and novelty accommodation. We then demonstrate how our SHELS representation and approach enable an iterative novelty detection and accommodation process in a class-incremental continual learning scenario without labeled class boundaries.\footnote{Source code supporting re-creation of all experiments can be found at \url{https://github.com/Lifelong-ML/SHELS}.}

\textbf{Datasets~~}  We evaluate our framework using MNIST~\citep{gradient1998lecun}, FashionMNIST~(FMNIST; \citealp{xiao2017fashion}), CIFAR10~\citep{krizhevsky2009learning}, SVHN~\citep{netzer2011reading}, and GTSRB~\citep{Stallkamp2012}.  MNIST, FMNIST, CIFAR10, and SVHN each consist of 10 classes, while GTSRB consists of 43 classes.

\textbf{Networks and training details~~} All experiments use DNN architectures that consist of ReLU convolutional layers followed by ReLU fully-connected layers and cosine normalization. For MNIST, FMNIST, CIFAR10 (within-dataset), SVHN, and GTSRB, the architecture consists of six convolutional layers followed by one or more linear layers and a final cosine normalization layer. \Revision{Across-datasets experiments with CIFAR10 as the ID classes use VGG16 pretrained on ImageNet, replacing the final fully-connected layer with a cosine normalization layer. We use the Adam optimizer for all datasets except for CIFAR10 across datasets, for which we use SGD as it achieves higher ID performance}. Implementation details, architectures, and hyperparameters can be found in Appendix~\ref{sec: implementation_details}.  Additionally, in all experiments, we compute thresholds and tune hyperparameters using only ID data.

\subsection{Novelty detection}
\label{sec:detect-exp}
We consider two experimental setups of varying difficulties to evaluate our method.  The standard evaluation for OOD detection in prior work considers novelty detection \textit{across datasets}, where ID data differs significantly from OOD data. In this case, one entire dataset is considered ID, and an entirely different dataset is selected as OOD. Additionally, we propose evaluating novelty detection \textit{within datasets}, where ID and OOD data share similar characteristics, requiring novelty detection approaches to learn more fine-grained differences than in the across-datasets experiments. For this case, we randomly sample $C$ classes from a single dataset as ID, and use the remaining classes from the dataset as OOD.  For MNIST, FMNIST, CIFAR10, and SVHN, we sample $C = 5$ ID classes and use the remaining $5$ as novel. For GTSRB, we sample $C = 23$ and use the remaining $20$ classes as OOD data.

In all experiments, we train a classification model over an ID train and validation set. We then evaluate the model on both ID test data and OOD test data. Using this procedure, we compare our OOD detection approach (Section~\ref{sec:novelty-detection}) to a baseline of the best performing state-of-the-art OOD detector~\citep{techapanurak2020hyperparameter}, which outperforms alternative methods including ODIN and Mahalanobis. We evaluate performance using the following measures:
\vspace{-0.6em}
\begin{itemize}[noitemsep,nolistsep,topsep=0pt,parsep=0pt,partopsep=0pt]
    \item OOD detection accuracy---the ability to correctly classify an OOD data point as novel
    \item ID classification accuracy---the ability to correctly determine the class of an ID data point
    \item Combined ID and OOD  accuracy---the ability to classify a data point as the correct ID class or as OOD
    \item AUROC---the ability to discriminate between ID and OOD data over the combined ID and OOD datasets.
\end{itemize}
\vspace{-0.6em}
To aid generalization of results, all experiments are repeated over $10$ trials with random seeds, within-dataset experiments use different sets of ID classes for each seed, and we perform statistical significance analyses over the trials.

\newcommand{\barplotwidth}{0.30\columnwidth}
\newcommand{\barplotheight}{0.25\columnwidth}

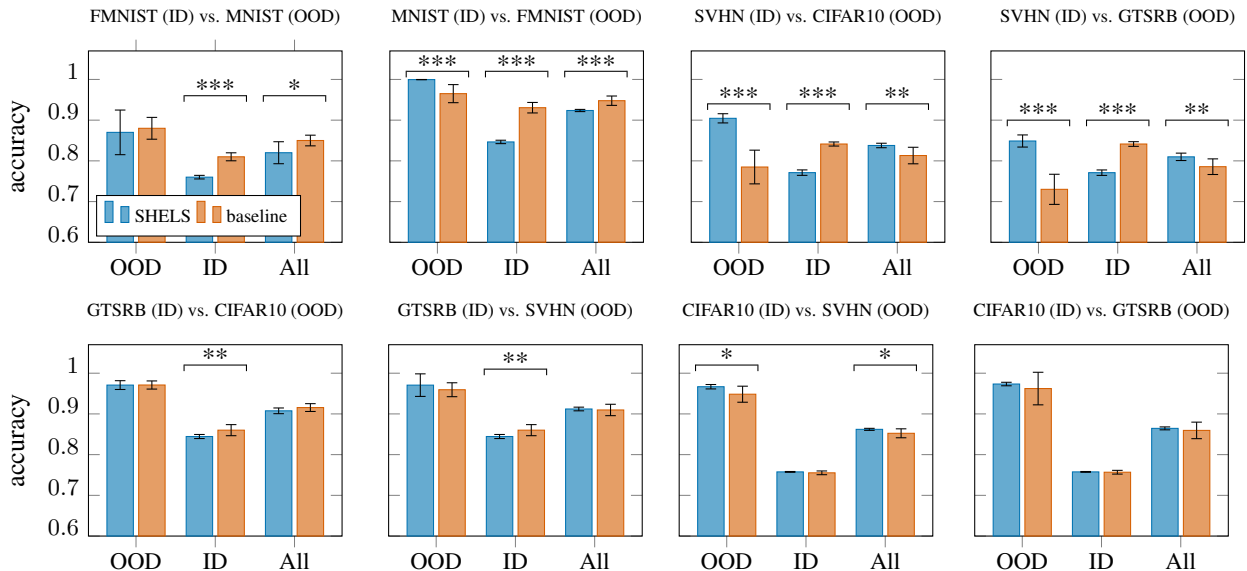
\begin{figure}[t]
\centering
    \begin{tikzpicture}
    \begin{axis} [
        ybar,
        width=\barplotwidth,
        height=\barplotheight,
        bar width=10pt,
        enlarge x limits=0.3,
        ymin=0.6,
        ymax=1.07,
        ytick={.6,.7,.8,.9,1.0},
        ylabel={accuracy},
        symbolic x coords={OOD, ID, All},
        ticklabel style = {font=\footnotesize, yshift=0.5ex},
        xtick=data,
        typeset ticklabels with strut,
        legend style={cells={anchor=west},
        font=\scriptsize},
        legend columns=4,
        legend pos=south west,
        title={FMNIST (ID) vs. MNIST (OOD)}, title style={font=\scriptsize},
    ]
    \addplot+ [
        draw = color1,
        fill = color1,
        fill opacity = 0.6,
        error bars/.cd,
            y dir=both,
            y explicit,
            error bar style={color=black},
    ] coordinates {
        (OOD,.87) +- (0,0.0548)
        (ID,.76) +- (0,0.0044)
        (All,.82) +- (0,0.027)
    };
    \addplot+ [
        draw = color3,
        fill = color3,
        fill opacity = 0.6,
        error bars/.cd,
            y dir=both,
            y explicit,
            error bar style={color=black},
    ]coordinates {
        (OOD,.88) +- (0,.0268)
        (ID,.81) +- (0,0.0098)
        (All,.85) +- (0,0.0131)
    };
    \draw [arrows={Bar[left]-Bar[right]}]
       ($(axis cs:ID,0.96)+(-12pt,0pt)$) -- node[midway, above=-4pt]{***} ($(axis cs:ID,0.96)+(12pt,0pt)$);
    \draw [arrows={Bar[left]-Bar[right]}]
       ($(axis cs:All,0.96)+(-12pt,0pt)$) -- node[midway, above=-4pt]{*} ($(axis cs:All,0.96)+(12pt,0pt)$);
    \legend{SHELS, baseline}
    \end{axis}
    \end{tikzpicture}
    \hfill
    \begin{tikzpicture}
    \begin{axis} [
        ybar,
        width=\barplotwidth,
        height=\barplotheight,
        bar width=10pt,
        enlarge x limits=0.3,
        ymin=0.6,
        ymax=1.07,
        ytick={.6,.7,.8,.9,1.0},
        yticklabels=\empty,
        symbolic x coords={ OOD,ID, All},
        ticklabel style = {font=\footnotesize, yshift=0.5ex},
        xtick=data,
        xtick pos=left,
        typeset ticklabels with strut,
        legend style={cells={anchor=west},
        font=\scriptsize},
        legend columns=4,
        legend pos=south west,
        title={MNIST (ID) vs. FMNIST (OOD)}, title style={font=\scriptsize},
    ]
    \addplot+ [
        draw = color1,
        fill = color1,
        fill opacity = 0.6,
        error bars/.cd,
            y dir=both,
            y explicit,
            error bar style={color=black},
    ] coordinates {
        (OOD,.9996) +- (0,.0003)
        (ID, .8465) +- (0,.0042)
        (All,.9237) +- (0,.0027)
    };
    \addplot+ [
        draw = color3,
        fill = color3,
        fill opacity = 0.6,
        error bars/.cd,
            y dir=both,
            y explicit,
            error bar style={color=black},
    ]coordinates {
        (OOD,.9651) +-  (0,.0222)
        (ID,.9306) +- (0,.0129)
        (All,.9478) +- (0,.0115)
    };
    \draw [arrows={Bar[left]-Bar[right]}]
       ($(axis cs:OOD,1.02)+(-12pt,0pt)$) -- node[midway, above=-4pt]{***} ($(axis cs:OOD,1.02)+(12pt,0pt)$);
    \draw [arrows={Bar[left]-Bar[right]}]
       ($(axis cs:ID,1.02)+(-12pt,0pt)$) -- node[midway, above=-4pt]{***} ($(axis cs:ID,1.02)+(12pt,0pt)$);
    \draw [arrows={Bar[left]-Bar[right]}]
       ($(axis cs:All,1.02)+(-12pt,0pt)$) -- node[midway, above=-4pt]{***} ($(axis cs:All,1.02)+(12pt,0pt)$);
    \end{axis}
    \end{tikzpicture}
    \hfill
    \begin{tikzpicture}
    \begin{axis} [
        ybar,
        width=\barplotwidth,
        height=\barplotheight,
        bar width=10pt,
        enlarge x limits=0.3,
        ymin=0.6,
        ymax=1.07,
        ytick={.6,.7,.8,.9,1.0},
        yticklabels=\empty,
        symbolic x coords={ OOD, ID, All},
        ticklabel style = {font=\footnotesize, yshift=0.5ex},
        xtick=data,
        xtick pos=left,
        typeset ticklabels with strut,
        legend style={cells={anchor=west},
        font=\scriptsize},
        legend columns=4,
        legend pos=south west,
        title={SVHN (ID) vs. CIFAR10 (OOD)}, title style={font=\scriptsize},
    ]
    \addplot+ [
        draw = color1,
        fill = color1,
        fill opacity = 0.6,
        error bars/.cd,
            y dir=both,
            y explicit,
            error bar style={color=black},
    ] coordinates {
        (OOD, .9045) +- (0, .0113)
        (ID, .7711) +- (0, .0067)
        (All, .8378) +- (0, .0055)
    };
    \addplot+ [
        draw = color3,
        fill = color3,
        fill opacity = 0.6,
        error bars/.cd,
            y dir=both,
            y explicit,
            error bar style={color=black},
    ]coordinates {
        (OOD,.7849) +- (0, .0414)
        (ID,.8414) +- (0, .005)
        (All,.8131) +- (0, .0203)
    };
    \draw [arrows={Bar[left]-Bar[right]}]
       ($(axis cs:OOD,0.94)+(-12pt,0pt)$) -- node[midway, above=-4pt]{***} ($(axis cs:OOD,0.94)+(12pt,0pt)$);
    \draw [arrows={Bar[left]-Bar[right]}]
       ($(axis cs:ID,0.94)+(-12pt,0pt)$) -- node[midway, above=-4pt]{***} ($(axis cs:ID,0.94)+(12pt,0pt)$);
    \draw [arrows={Bar[left]-Bar[right]}]
       ($(axis cs:All,0.94)+(-12pt,0pt)$) -- node[midway, above=-4pt]{**} ($(axis cs:All,0.94)+(12pt,0pt)$);
    \end{axis}
    \end{tikzpicture}
    \hfill
    \begin{tikzpicture}
    \begin{axis} [
        ybar,
        width=\barplotwidth,
        height=\barplotheight,
        bar width=10pt,
        enlarge x limits=0.3,
        ymin=0.6,
        ymax=1.07,
        ytick={.6,.7,.8,.9,1.0},
        yticklabels=\empty,
        symbolic x coords={OOD, ID, All},
        ticklabel style = {font=\footnotesize, yshift=0.5ex},
        xtick=data,
        xtick pos=left,
        typeset ticklabels with strut,
        legend style={cells={anchor=west},
        font=\scriptsize},
        legend columns=4,
        legend pos=south west,
        title={SVHN (ID) vs. GTSRB (OOD)}, title style={font=\scriptsize},
    ]
    \addplot+ [
        draw = color1,
        fill = color1,
        fill opacity = 0.6,
        error bars/.cd,
            y dir=both,
            y explicit,
            error bar style={color=black},
    ] coordinates {
        (OOD,.8488) +- (0, .0149)
        (ID,.7710) +- (0, .0067)
        (All,.8099) +- (0, .0091)
    };
    \addplot+ [
        draw = color3,
        fill = color3,
        fill opacity = 0.6,
        error bars/.cd,
            y dir=both,
            y explicit,
            error bar style={color=black},
    ]coordinates {
        (OOD,.7301) +- (0, .0370)
        (ID,.8414) +- (0, .0057)
        (All,.7857) +- (0, .0192)
    };
    \draw [arrows={Bar[left]-Bar[right]}]
       ($(axis cs:OOD,0.9)+(-12pt,0pt)$) -- node[midway, above=-4pt]{***} ($(axis cs:OOD,0.9)+(12pt,0pt)$);
    \draw [arrows={Bar[left]-Bar[right]}]
       ($(axis cs:ID,0.9)+(-12pt,0pt)$) -- node[midway, above=-4pt]{***} ($(axis cs:ID,0.9)+(12pt,0pt)$);
    \draw [arrows={Bar[left]-Bar[right]}]
       ($(axis cs:All,0.9)+(-12pt,0pt)$) -- node[midway, above=-4pt]{**} ($(axis cs:All,0.9)+(12pt,0pt)$);
    \end{axis}
    \end{tikzpicture}
    \\
    \begin{tikzpicture}
    \begin{axis} [
        ybar,
        width=\barplotwidth,
        height=\barplotheight,
        bar width=10pt,
        enlarge x limits=0.3,
        ymin=0.6,
        ymax=1.07,
        ytick={.6,.7,.8,.9,1.0},
        ylabel={accuracy},
        symbolic x coords={ OOD, ID, All},
        ticklabel style = {font=\footnotesize, yshift=0.5ex},
        xtick=data,
        xtick pos=left,
        typeset ticklabels with strut,
        legend style={cells={anchor=west},
        font=\scriptsize},
        legend columns=4,
        legend pos=south west,
        title={GTSRB (ID) vs. CIFAR10 (OOD)}, title style={font=\scriptsize},
    ]
    \addplot+ [
        draw = color1,
        fill = color1,
        fill opacity = 0.6,
        error bars/.cd,
            y dir=both,
            y explicit,
            error bar style={color=black},
    ] coordinates {
        (OOD,.9708) +- (0, .0108)
        (ID,.8445) +- (0, .0050)
        (All,.9076) +- (0, .0070)
    };
    \addplot+ [
        draw = color3,
        fill = color3,
        fill opacity = 0.6,
        error bars/.cd,
            y dir=both,
            y explicit,
            error bar style={color=black},
    ]coordinates {
        (OOD,.9710) +- (0, .0100)
        (ID,.8601) +- (0, .0136)
        (All,.9156) +- (0, .0096)
    };
    \draw [arrows={Bar[left]-Bar[right]}]
       ($(axis cs:ID,1.02)+(-12pt,0pt)$) -- node[midway, above=-4pt]{**} ($(axis cs:ID,1.02)+(12pt,0pt)$);
    \end{axis}
    \end{tikzpicture}
    \hfill
    \begin{tikzpicture}
    \begin{axis} [
        ybar,
        width=\barplotwidth,
        height=\barplotheight,
        bar width=10pt,
        enlarge x limits=0.3,
        ymin=0.6,
        ymax=1.07,
        ytick={.6,.7,.8,.9,1.0},
        yticklabels=\empty,
        symbolic x coords={OOD, ID, All},
        ticklabel style = {font=\footnotesize, yshift=0.5ex},
        xtick=data,
        xtick pos=left,
        typeset ticklabels with strut,
        legend style={cells={anchor=west},
        font=\scriptsize},
        legend columns=4,
        legend pos=south west,
        title={GTSRB (ID) vs. SVHN (OOD)}, title style={font=\scriptsize},
    ]
    \addplot+ [
        draw = color1,
        fill = color1,
        fill opacity = 0.6,
        error bars/.cd,
            y dir=both,
            y explicit,
            error bar style={color=black},
    ] coordinates {
        (OOD, .9707) +- (0, .0277)
        (ID, .8445) +- (0, .0050)
        (All, .9121) +- (0, .0045)
    };
    \addplot+ [
        draw = color3,
        fill = color3,
        fill opacity = 0.6,
        error bars/.cd,
            y dir=both,
            y explicit,
            error bar style={color=black},
    ]coordinates {
        (OOD, .9594) +- (0, .0171)
        (ID, .8601) +- (0, .0136)
        (All, .9097) +- (0, .0140)
    };
    \draw [arrows={Bar[left]-Bar[right]}]
       ($(axis cs:ID,1.01)+(-12pt,0pt)$) -- node[midway, above=-4pt]{**} ($(axis cs:ID,1.01)+(12pt,0pt)$);
    \end{axis}
    \end{tikzpicture}
    \hfill
    \begin{tikzpicture}
    \begin{axis} [
        ybar,
        width=\barplotwidth,
        height=\barplotheight,
        bar width=10pt,
        enlarge x limits=0.3,
        ymin=0.6,
        ymax=1.07,
        ytick={.6,.7,.8,.9,1.0},
        yticklabels=\empty,
        symbolic x coords={OOD, ID, All},
        ticklabel style = {font=\footnotesize, yshift=0.5ex},
        xtick=data,
        xtick pos=left,
        typeset ticklabels with strut,
        legend style={cells={anchor=west},
        font=\scriptsize},
        legend columns=4,
        legend pos=south west,
        title={CIFAR10 (ID) vs. SVHN (OOD)}, title style={font=\scriptsize},
    ]
    \addplot+ [
        draw = color1,
        fill = color1,
        fill opacity = 0.6,
        error bars/.cd,
            y dir=both,
            y explicit,
            error bar style={color=black},
    ] coordinates {
        (OOD, .9668) +- (0, .0054)
        (ID, .7576) +- (0, .0010)
        (All, .8621) +- (0, .0025)
    };
    \addplot+ [
        draw = color3,
        fill = color3,
        fill opacity = 0.6,
        error bars/.cd,
            y dir=both,
            y explicit,
            error bar style={color=black},
    ]coordinates {
        (OOD, .9484) +- (0, .0197)
        (ID, .7553) +- (0, .0046)
        (All, .8524) +- (0, .0110)
    };
    \draw [arrows={Bar[left]-Bar[right]}]
       ($(axis cs:OOD,1.02)+(-12pt,0pt)$) -- node[midway, above=-4pt]{*} ($(axis cs:OOD,1.02)+(12pt,0pt)$);
    \draw [arrows={Bar[left]-Bar[right]}]
       ($(axis cs:All,1.02)+(-12pt,0pt)$) -- node[midway, above=-4pt]{*} ($(axis cs:All,1.02)+(12pt,0pt)$);
    \end{axis}
    \end{tikzpicture}
    \hfill
    \begin{tikzpicture}
    \begin{axis} [
        ybar,
        width=\barplotwidth,
        height=\barplotheight,
        bar width=10pt,
        enlarge x limits=0.3,
        ymin=0.6,
        ymax=1.07,
        ytick={.6,.7,.8,.9,1.0},
        yticklabels=\empty,
        symbolic x coords={OOD, ID, All},
        ticklabel style = {font=\footnotesize, yshift=0.5ex},
        xtick=data,
        xtick pos=left,
        typeset ticklabels with strut,
        legend style={cells={anchor=west},
        font=\scriptsize},
        legend columns=4,
        legend pos=south west,
        title={CIFAR10 (ID) vs. GTSRB (OOD)}, title style={font=\scriptsize},
    ]
    \addplot+ [
        draw = color1,
        fill = color1,
        fill opacity = 0.6,
        error bars/.cd,
            y dir=both,
            y explicit,
            error bar style={color=black},
    ] coordinates {
        (OOD, .9734) +- (0, .0043)
        (ID, .7576) +- (0, .0010)
        (All, .8645) +- (0, .0036)
    };
    \addplot+ [
        draw = color3,
        fill = color3,
        fill opacity = 0.6,
        error bars/.cd,
            y dir=both,
            y explicit,
            error bar style={color=black},
    ]coordinates {
        (OOD, .9623) +- (0, .0400)
        (ID, .7568) +- (0, .0046)
        (All, .8595) +- (0, .0203)
    };
    \end{axis}
    \end{tikzpicture}
    \vspace{-0.35cm}
    \caption{\label{fig:OOD-detec-acc-across}Novelty detection across datasets, showing OOD detection accuracy (OOD), ID classification accuracy (ID), and combined ID and OOD classification accuracy (ALL) with mean $\pm$ standard deviation. Statistically significant differences (after Holm-Bonferroni adjustment) are denoted for $p<.05$ (*), $p<.01$ (**), and $p<.001$ (***). }
\end{figure}

Detection and classification accuracies for across-datasets experiments are shown in Figure~\ref{fig:OOD-detec-acc-across}, with corresponding AUROC measures given in Table~\ref{tab:auroc-across}.  The results show a trade-off between the learnt models' abilities to perform OOD detection and maintain  ID classification accuracy.  To analyze this trade-off, we performed two-tailed unpaired Student's t-tests between the baseline and our SHELS approach, using a Holm-Bonferroni adjustment to account for testing multiple measures. All test statistics, p-values, and corrections are reported in Appendix~\ref{sec:stat-test-details}.  The t-tests show that SHELS significantly improves OOD performance with respect to the baseline in half of the cases---MNIST-FMNIST, SVHN-CIFAR10, SVHN-GTSRB, CIFAR10-SVHN---with comparable performance otherwise.  This improvement in OOD detection is further supported by the AUROCs (Table~\ref{tab:auroc-across}), which shows that SHELS outperforms the baseline over a majority of experiments. In contrast, t-tests show the baseline outperforms our method in ID accuracy in all cases except when CIFAR10 is used as the ID data, where performance is comparable.  Note that the combined metric generally shows comparable performance between the methods, exemplifying the nature of the OOD vs.~ID tradeoff.  As such, while SHELS significantly improves OOD detection at the cost of some ID accuracy, it has comparable overall performance to the state-of-the-art OOD detection baseline on the relatively easier across-datasets experiments.

Detection and classification accuracies for the more challenging within-dataset experiments are shown in Figure~\ref{fig:OOD_detec_within}, with corresponding AUROC measures given in Table~\ref{tab:auroc-within}.  The results show a similar OOD vs.~ID accuracy trade-off as in the across-datasets experiments, although the differences between SHELS and the baseline become more apparent due to the more difficult within-dataset case. \Revision{Paired t-tests showed significant improvements in OOD detection using SHELS for four of the five datasets, with no significant difference found for GTSRB}.  The AUROC measure further supports SHELS's improved OOD detection performance, outperforming the baseline by large margins for many of the datasets.  In contrast, paired t-tests showed significant improvements in ID classification accuracy using the baseline approach for all datasets, \Revision{although SHELS exhibits significant improvements in combined accuracy for four of the five datasets, with no significant difference found for GTSRB}. We conclude that our approach's improvements in OOD detection accuracy result in comparable or significant improvements in overall performance compared to the state of the art.

\subsection{Novelty accommodation}
\label{sec:accom-experiment}

\begin{table}[t]
    \small
    \begin{minipage}[t]{0.55\linewidth}
    \centering
    \captionsetup{width=0.95\linewidth}
    \caption{AUROC of OOD detection across datasets (mean $\pm$ standard deviation over 10 trials).}
    \vspace{-.25cm}
    \label{tab:auroc-across}
        \begin{tabular}{ll|llc}
         ID Data & OOD Data & SHELS &Baseline & Sig.\\
        \hline
        MNIST  & FMNIST          & \textbf{0.99} {\scriptsize $\pm$ 0.002 } & 0.94 {\scriptsize $\pm$ 0.03} & ***\\

        FMNIST  & MNIST          & \textbf{0.88} {\scriptsize $\pm$ 0.063 } & 0.77 {\scriptsize $\pm$ 0.064} & **\\

        GTSRB  & CIFAR10          &\textbf{0.96} {\scriptsize $\pm$ 0.01 } &0.90 {\scriptsize $\pm$ 0.032} & ***\\
        GTSRB  & SVHN          & \textbf{0.98 }{\scriptsize $\pm$ 0.005 } & 0.94 {\scriptsize $\pm$ 0.018} & ***\\

        SVHN  & CIFAR10          & 0.90 {\scriptsize $\pm$ 0.007 } & 0.90 {\scriptsize $\pm$ 0.01} & \\
        SVHN  & GTSRB          & 0.88 {\scriptsize $\pm$ 0.009 } & \textbf{0.89} {\scriptsize $\pm$ 0.008} & *\\

        CIFAR10  & GTSRB          & \textbf{0.93} {\scriptsize $\pm$ 0.002 } & 0.92 {\scriptsize $\pm$ 0.025} & \\
        CIFAR10  & SVHN          & \textbf{0.97} {\scriptsize $\pm$ 0.003 } & 0.95 {\scriptsize $\pm$ 0.05} & \\\hline
        \end{tabular}
    \end{minipage}%
    \begin{minipage}[t]{0.45\linewidth}
    \centering
    \captionsetup{width=0.95\linewidth}
    \caption{AUROC of OOD detection within datasets (mean $\pm$ std.~dev.~over 10 trials).}
    \vspace{-.25cm}
    \label{tab:auroc-within}
        \begin{tabular}{l|llc}
        Dataset & SHELS & Baseline & Sig.
        \\ \hline
        MNIST            & \textbf{0.96} {\scriptsize $\pm$ 0.017 } & 0.91 {\scriptsize $\pm$ 0.062} & *\\

        FMNIST     & \textbf{0.79} {\scriptsize $\pm$ 0.063} & 0.74 {\scriptsize $\pm$ 0.084} & \\

        GTSRB            & \textbf{0.88 }{\scriptsize $\pm$ 0.055 } & 0.81 {\scriptsize $\pm$ 0.061} & * \\

        \Revision{CIFAR10}          & \Revision{\textbf{0.70} {\scriptsize $\pm$ 0.059 } }
        & \Revision{0.69 {\scriptsize $\pm$ 0.067 }} &  \\

        SVHN      & \textbf{0.82 }{\scriptsize $\pm$ 0.023} & 0.80 {\scriptsize $\pm$ 0.018} & \\\hline
        \end{tabular}
    \end{minipage}
    \vspace{-0.2cm}
\end{table}

We evaluate SHELS's ability to accommodate novel classes through class-incremental continual learning experiments on the MNIST, \Revision{FMNIST (see Appendix~\ref{sec:additionaleClExp}), and GTSRB (see Appendix~\ref{sec:additionaleClExp}) datasets and compare it with AGS-CL~\citep{jung2020continual}, a similar sparsity and weight freezing approach that has shown robust performance, and GDUMB~\citep{prabhu2020gdumb}, a replay-based method.} We also demonstrate the importance of key components of SHELS that enable incremental novelty accommodation, including encouraging group sparsity and penalizing weight updates for novelty accommodation, via ablation studies.
We begin by training a classification model to learn SHELS features using the ID train set. For MNIST, the ID data consists of 7 randomly sampled classes and the remaining 3 classes are OOD. The OOD classes are introduced incrementally in the typical class-incremental continual learning fashion, and our algorithm accommodates them in the SHELS representation as described in Section~\ref{sec:novelty-accomm}.

\begin{figure}[t]
    \centering
    \begin{tikzpicture}
    \begin{axis} [
        ybar,
        height=0.3\columnwidth,
        width=0.85\columnwidth,
        bar width=7pt,
        enlarge x limits=0.15,
        ymin=0.3,
        ymax=1.1,
        ytick={.4,.6,.8,1.0},
        ylabel={accuracy},
        symbolic x coords={MNIST, FMNIST, SVHN, CIFAR10, GTSRB},
        ticklabel style = {font=\footnotesize, yshift=0.5ex},
        xtick=data,
        xtick pos=left,
        typeset ticklabels with strut,
        legend style={cells={anchor=west}, font=\scriptsize},
        legend pos=outer north east,
    ]
    \addplot+ [
        draw = color1,
        fill = color1,
        fill opacity = 0.6,
        error bars/.cd,
            y dir=both,
            y explicit,
            error bar style={color=black},
    ] coordinates {
        (MNIST, .9672) +- (0, .0185)
        (FMNIST, .6492) +- (0, .1398)
        (SVHN, .6925) +- (0, .0450)
        (CIFAR10, .5978) +- (0, .0658)
        (GTSRB, .8143) +- (0, .0433)
    };
    \addplot+ [
        draw = color3,
        fill = color3,
        fill opacity = 0.6,
        error bars/.cd,
            y dir=both,
            y explicit,
            error bar style={color=black},
    ] coordinates {
        (MNIST, .7765) +- (0, .0788)
        (FMNIST, .4364) +- (0, .1211)
        (SVHN, .5192) +- (0, .0673)
        (CIFAR10, .5261) +- (0, .0513)
        (GTSRB, .7655) +- (0, .1238)
    };
    \addplot+ [
        draw = color1,
        fill opacity = 1,
        pattern = crosshatch,
        pattern color = color1!30!color1,
        error bars/.cd,
            y dir=both,
            y explicit,
            error bar style={color=black},
    ] coordinates {
        (MNIST, .8513) +- (0, .0041)
        (FMNIST, .7767) +- (0, .0256)
        (SVHN, .7922) +- (0, .0116)
        (CIFAR10, .6713) +- (0, .027)
        (GTSRB, .8403) +- (0, .0133)
    };
    \addplot+ [
        draw=color3,
        fill opacity = 1,
        pattern = crosshatch,
        pattern color = color3!30!color3,
        error bars/.cd,
            y dir=both,
            y explicit,
            error bar style={color=black},
    ] coordinates {
        (MNIST, .9258) +- (0, .0008)
        (FMNIST, .8434) +- (0, .0419)
        (SVHN, .8749) +- (0, .0195)
        (CIFAR10, .6782) +- (0, .0303)
        (GTSRB, .8871) +- (0, .0190)
    };
    \addplot+ [
        draw=color1,
        fill opacity = 1,
        pattern = horizontal lines,
        pattern color = color1!30!color1,
        error bars/.cd,
            y dir=both,
            y explicit,
            error bar style={color=black},
    ] coordinates {
        (MNIST, .9092) +- (0, .0084)
        (FMNIST, .7129) +- (0, .0645)
        (SVHN, .7423) +- (0, .0215)
        (CIFAR10, .6345) +- (0, .0418)
        (GTSRB, .8273) +- (0, .0214)
    };
    \addplot+ [
        draw=color3,
        fill opacity = 1,
        pattern=horizontal lines,
        pattern color = color3!30!color3,
        error bars/.cd,
            y dir=both,
            y explicit,
            error bar style={color=black},
    ] coordinates {
        (MNIST, .8512) +- (0, .0383)
        (FMNIST, .6399) +- (0, .0545)
        (SVHN, .6971) +- (0, .0263)
        (CIFAR10, .6022) +- (0, .0364)
        (GTSRB, .8253) +- (0, .0691)
    };
    \draw [arrows={Bar[left]-Bar[right]}]
       ($(axis cs:MNIST,1.03)+(-26pt,0pt)$) -- node[midway, above=-4pt]{***} ($(axis cs:MNIST,1.03)+(-10pt,0pt)$);
    \draw [arrows={Bar[left]-Bar[right]}]
       ($(axis cs:MNIST,1.03)+(-8pt,0pt)$) -- node[midway, above=-4pt]{***} ($(axis cs:MNIST,1.03)+(8pt,0pt)$);
    \draw [arrows={Bar[left]-Bar[right]}]
       ($(axis cs:MNIST,1.03)+(10pt,0pt)$) -- node[midway, above=-4pt]{**} ($(axis cs:MNIST,1.03)+(26pt,0pt)$);
    \draw [arrows={Bar[left]-Bar[right]}]
       ($(axis cs:FMNIST,0.94)+(-26pt,0pt)$) -- node[midway, above=-4pt]{**} ($(axis cs:FMNIST,0.94)+(-10pt,0pt)$);
    \draw [arrows={Bar[left]-Bar[right]}]
       ($(axis cs:FMNIST,0.94)+(-8pt,0pt)$) -- node[midway, above=-4pt]{***} ($(axis cs:FMNIST,0.94)+(8pt,0pt)$);
    \draw [arrows={Bar[left]-Bar[right]}]
       ($(axis cs:FMNIST,0.94)+(10pt,0pt)$) -- node[midway, above=-4pt]{*} ($(axis cs:FMNIST,0.94)+(26pt,0pt)$);
    \draw [arrows={Bar[left]-Bar[right]}]
       ($(axis cs:SVHN,0.95)+(-26pt,0pt)$) -- node[midway, above=-4pt]{***} ($(axis cs:SVHN,0.95)+(-10pt,0pt)$);
    \draw [arrows={Bar[left]-Bar[right]}]
       ($(axis cs:SVHN,0.95)+(-8pt,0pt)$) -- node[midway, above=-4pt]{***} ($(axis cs:SVHN,0.95)+(8pt,0pt)$);
    \draw [arrows={Bar[left]-Bar[right]}]
       ($(axis cs:SVHN,0.95)+(10pt,0pt)$) -- node[midway, above=-4pt]{***} ($(axis cs:SVHN,0.95)+(26pt,0pt)$);
    \draw [arrows={Bar[left]-Bar[right]}]
       ($(axis cs:CIFAR10,0.9)+(-26pt,0pt)$) -- node[midway, above=-4pt]{***} ($(axis cs:CIFAR10,0.9)+(-10pt,0pt)$);
    \draw [arrows={Bar[left]-Bar[right]}]
       ($(axis cs:CIFAR10,0.9)+(-8pt,0pt)$) -- node[midway, above=-4pt]{**} ($(axis cs:CIFAR10,0.9)+(8pt,0pt)$);
    \draw [arrows={Bar[left]-Bar[right]}]
       ($(axis cs:CIFAR10,0.9)+(10pt,0pt)$) -- node[midway, above=-4pt]{***} ($(axis cs:CIFAR10,0.9)+(26pt,0pt)$);
    \draw [arrows={Bar[left]-Bar[right]}]
       ($(axis cs:GTSRB,0.96)+(-8pt,0pt)$) -- node[midway, above=-4pt]{***} ($(axis cs:GTSRB,0.96)+(8pt,0pt)$);
    \legend{OOD (SHELS), OOD (baseline), ID (SHELS), ID (baseline), All (SHELS), All (baseline)}
    \end{axis}
    \end{tikzpicture}
    \vspace{-.3cm}
    \caption{\label{fig:OOD_detec_within}Novelty detection within datasets, showing OOD detection accuracy (OOD), ID classification accuracy (ID), and combined ID and OOD classification accuracy (ALL) with mean $\pm$ standard deviation. Statistically significant differences (after Holm-Bonferroni adjustment) are denoted for $p<.05$ (*), $p<.01$ (**), and $p<.001$ (***).}
    \vspace{-0.1cm}
\end{figure}
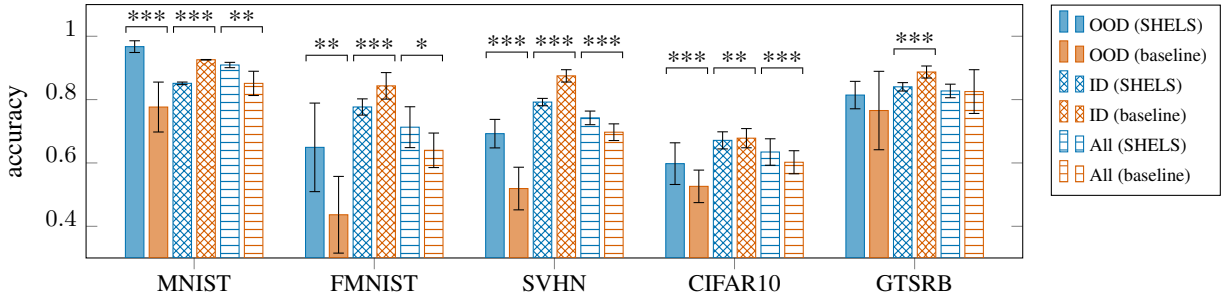

The class-incremental learning performance of SHELS is shown in the orange performance curves of Figure~\ref{cl_results}. When a new class is presented to the model, for instance class 8, it experiences a drop in test accuracy depicted by the solid orange curve at epoch 10, as the test set now includes data from the previously unseen class 8. The model then learns class 8, as shown by the dashed novel class test accuracy curve, while preserving performance on the previous seven classes, as evidenced by the upward trend of the overall test accuracy curve.  This trend repeats for subsequently introduced novel classes, showing that SHELS learns and accommodates novel classes while mitigating catastrophic forgetting. \Revision{SHELS}
\begin{wrapfigure}{r}{0.42\textwidth}
\vspace{-0.em} 
    \begin{minipage}{0.42\textwidth}
        \centering
        \includegraphics[width=1.0\columnwidth,clip, trim=.1in .2in 0in .2in ]{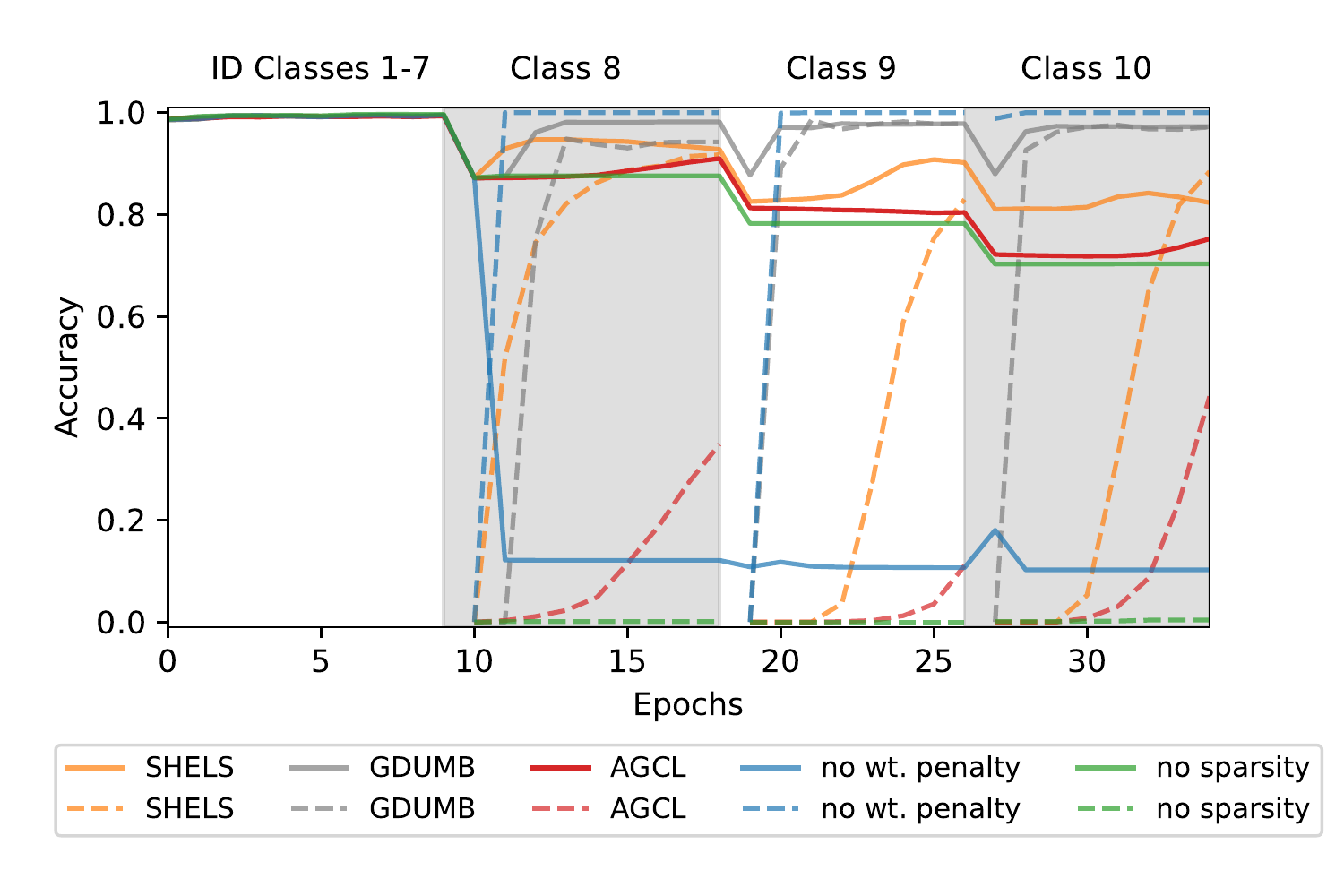}
        \caption{\label{cl_results}\Revision{Class incremental learning on MNIST.  Solid lines show overall test accuracy, dashes show test accuracy of the novel class only.}}
    \end{minipage}
    \vspace{-1em}
    \end{wrapfigure}
\Revision{significantly outperforms the state-of-the-art AGS-CL method. AGS-CL's performance is closer to SHELS without group sparsity (Ablation 1 below) than to the full SHELS implementation, indicating that SHELS induces sparsity for incorporating new knowledge better than AGS-CL. We also compare our approach to GDUMB, which uses $K$ replay samples balanced across the ID classes. For our experiments, we set $K$ to 1\% of the total training data. While GDUMB outperforms SHELS, it is important to note that SHELS does not need to store replay data, which may not be possible in memory-constrained or privacy-sensitive applications.  As $K$ increases, GDUMB serves as an upper bound on performance, representing the non-continual learning setting where all data is available to the training algorithm at all times. As the fundamental mechanism of SHELS is complementary to replay, we show that the ideas of SHELS and GDUMB are compatible in Appendix~\ref{sec:additionaleClExp}, combining them to yield a new SHELS-GDUMB algorithm}.

\textbf{Ablation 1: Group Sparsity~~} To show the importance of encouraging group sparsity for novelty accommodation, we evaluate our method by repeating the same experiment without group sparsity regularization, (\textit{no sparsity} curves in Figure~\ref{cl_results}). The regular drop in test accuracy when new classes are introduced and the new class performance staying at zero indicates that the network has no capacity available to accommodate new class information.

\textbf{Ablation 2: Penalizing Weight Updates~~} To show the effectiveness of penalizing weight updates for novelty accommodation, we evaluate our method without any penalties on weight updates (\textit{no wt.~penalty} curves in Figure~\ref{cl_results}).  While new classes can be learnt successfully, as shown by the new class performance, there is no mechanism to maintain performance on previous tasks, catastrophic forgetting occurs, and test accuracy drops significantly.

The ablations demonstrate the roles of group sparsity and weight update penalization in our method, enabling SHELS to accommodate novel classes in a class-incremental continual learning setting. We evaluate our continual learning approach over \Revision{additional datasets and longer class sequences} in Appendix~\ref{sec:additionaleClExp}, showing qualitatively similar results.

\subsection{Novelty detection and accommodation}
\label{sec:detect-accomm-exp}



In this section, we demonstrate the performance of SHELS to both detect and accommodate OOD data using the pipeline presented in Section~\ref{sec:detect_accomm}. To this end, we conducted a class-incremental continual learning experiment where \textit{class boundaries are not given to the learning agent}. The experiment is designed such that the agent must alternate between novelty detection and accommodation phases, and the decision to switch from detection to accommodation is made by the agent. At each epoch in the detection phase, the agent encounters a random sample of either ID or OOD data.  In our experiments, OOD data is always provided at the fourth detection epoch for consistent visualization---note that this consistency  has no effect on the agent's ability to detect OOD data. If the agent detects a sample as OOD with a confidence score greater than a preset threshold, it triggers an accommodation phase where it is provided with more of the data from the last detection epoch. The model accommodates this data as a novel class, and once the new class is learnt, it switches back to detection, and the detection/accommodation cycle continues. To the best of our knowledge, this is a novel evaluation paradigm for continual learning without labeled class boundaries. With no methods for this setting in the literature, we provide results from our approach as a case study \Revision{and compare against a combination of the OOD detection and continual learning components of \citet{lee2018simple}, which we denote MAHA}.

\begin{figure}[t]
    \centering
  \begin{subfigure}{0.5\textwidth}
      \includegraphics[width=\textwidth,clip, trim=.15in .25in .15in 3.15in]{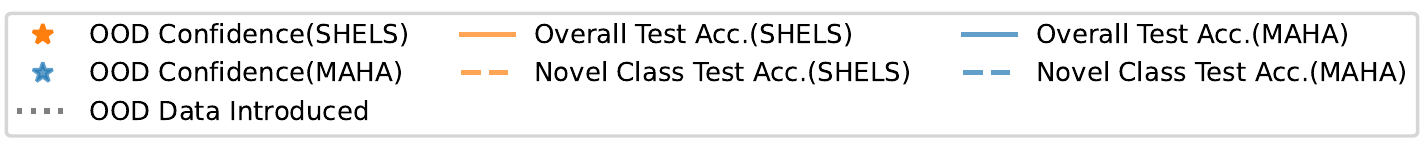}
  \end{subfigure}\\
  \begin{subfigure}{0.5\textwidth}
        \centering
        \includegraphics[height=1.5in,clip, trim=.15in 0.9in .15in .15in]{Image_pdfs/full_pipeline_mnist_with_baseline.pdf}
        \vspace{-0.5em}
        \caption{\Revision{MNIST}}
        \label{fig:full-pipeline-1}
    \end{subfigure}%
    \begin{subfigure}{0.5\textwidth}
        \centering
        \includegraphics[height=1.5in,clip, trim=.15in 0.9in .15in .15in]{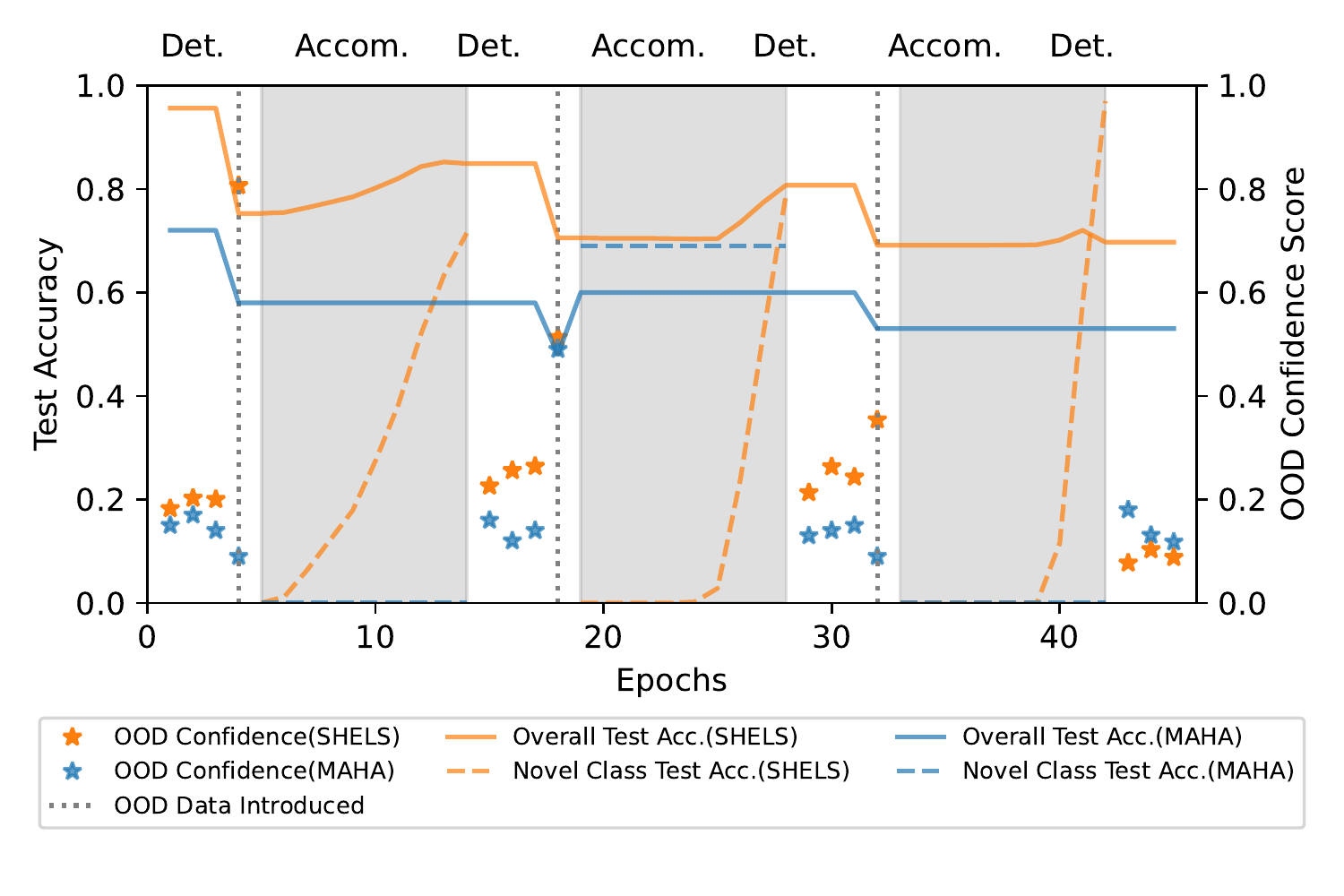}
        \vspace{-0.5em}
        \caption{\Revision{FMNIST}}
        \label{fig:full-pipeline_2}
    \end{subfigure}%
    \vspace{-0.5em}
    \caption{Class incremental learning without boundaries, alternating between detection and accommodation phases.}
    \label{fig:full-pipeline}
\end{figure}

We use seven randomly sampled classes from the MNIST dataset as ID and the remaining as OOD, and \Revision{four randomly sampled classes from FMNIST datatset as ID and three out of the remaining six classes as OOD. A performance comparison of our approach with MAHA is shown in Figure~\ref{fig:full-pipeline}}. For MNIST, we begin by learning a classification model consisting of SHELS features using the ID train set. The model is then deployed and it enters the detection phase. For the first three epochs, it encounters a random sample of test data from the previously seen seven classes, which it correctly identifies as ID as shown by the low OOD confidence scores. At the fourth epoch, the model encounters a sample of OOD class data, correctly identifies it, and switches to the accommodation phase. The newly detected class is accommodated (as shown by the novel class test accuracy curve) while maintaining performance on the previous classes (as shown by the overall test accuracy curve). The model switches back to the detection phase, and the process repeats until the final detection phase, where all ten classes have been accommodated and all data is correctly detected as ID.
\Revision{We repeat the same procedure for FMNIST. For both datasets, SHELS correctly identifies and accommodates all novel classes. In comparison, while MAHA is able to detect and accommodate all novel MNIST classes, it fails to detect two of the novel FMNIST classes, resulting in significantly lower performance when compared to SHELS.  

We offer the following explanation for why SHELS outperforms MAHA. For novelty detection, MAHA computes a confidence score using the Mahalanobis distance with respect to the nearest ID class-conditional
distribution. While this approach identifies novel classes with features that significantly differ from the ID classes, it fails to identify novel classes with some features similar to the ID classes, which is a frequent case within the FMNIST dataset. SHELS overcomes this drawback by computing the similarity between the high-level activation patterns of the novel class and the exclusive feature sets for each ID class. Since each ID class is represented by a unique set of high-level features, then even if a novel class shares some features with the ID classes, comparing the similarity between these high-level exclusive feature sets still results in effective novelty detection.
As for novelty accommodation, MAHA reuses the previously learnt representation, and thus lacks the ability to learn new knowledge. For both the MNIST and FMNIST datasets, MAHA performs poorly when compared to SHELS, since new knowledge needs to be learnt when accommodating new classes. In contrast, due to the sparsity, SHELS is able to learn new knowledge and accommodate new classes more effectively.}

The above experiments are analogous to an agent deployed in the open world, where it encounters objects from known classes in its environment. Eventually, the agent will encounter an object of an unknown class, detect that the object is novel, gather additional data to accommodate the new object in its model, and continue to operate in its environment with the updated set of ID classes. Our results validate that the SHELS representation enables integrated novelty detection and accommodation within a single model, and can operate effectively in such an open-world setting.

\vspace{-.1cm}
\section{Conclusion}
\label{sec:conclusion}
\vspace{-.1cm}

In this work, we show the potential of combining novelty detection with continual learning to develop open-world learners. We demonstrated a system capable of class-incremental continual learning without class boundaries, which automatically detects the presence of new classes via a novel OOD detection method that leverages our human-inspired SHELS representation and significantly outperforms the state of the art. 
In particular, SHELS leverages the \textit{exclusivity} of high-level feature sets to detect OOD data, and the \textit{sharing} of low-level features to free up nodes for accommodating future classes.
Our experimental results illustrate the trade-off between ID classification and OOD detection accuracies. To mitigate this trade-off, one promising avenue of future work is to develop alternative methods for encouraging exclusive feature sets. 
Our method can also be extended by integrating complementary continual learning approaches to increase learning robustness. In particular, dynamic network expansion would endow agents with the capability to automatically grow the capacity of the SHELS network if it runs out of unused nodes to accommodate large numbers of new classes, enabling potentially unbounded learning. With a suitably robust system, we plan to evaluate our approach in physical open-world environments through situated robot deployments as future work.

\subsubsection*{Acknowledgments}
We would like to thank the anonymous reviewers for their helpful feedback on the earlier draft. The research presented in this paper was partially supported by the DARPA Lifelong Learning Machines program under grant FA8750-18-2-0117, the DARPA SAIL-ON program under contract HR001120C0040, the DARPA ShELL program under agreement HR00112190133, and the Army Research Office under MURI grant W911NF20-1-0080.

\bibliography{collas2022_conference}
\bibliographystyle{collas2022_conference}

\newpage
\appendix

\section{Orthogonal features are exclusive sets}

\label{sec:ProofOfPropOrthogonalityCapturesExclusivity}

This section formalizes the notation that orthogonal features are exclusive sets. Let each data instance $\bm{x} \in \Reals^d$ be embedded as set of derived features $\bm{f} = g(\bm{x}) \in \Reals^k$. Recall that our goal is for each class to have an exclusive embedding of the non-zero entries of $\bm{f}$.

\begin{proposition}
\label{prop:OrthogonalityCapturesExclusivity}
Let $\bm{f}_1, \ldots, \bm{f}_C \in \Reals^k$ be the orthogonal embeddings for each of $C$ classes, such that $\forall i\neq j \ \bm{f}_i \perp \bm{f}_j$, with $k \geq C$. Then, there exists an orthogonal transformation $T$ such that the transformed embeddings $T(\bm{f}_1), \ldots, T(\bm{f}_C)$ form a collection of exclusive sets.
\end{proposition}
\begin{proof}
Since $\bm{f}_1, \ldots, \bm{f}_C \in \Reals^k$ are orthogonal, they are linearly independent and form a basis for a subspace $\Reals^C \subseteq \Reals^k$. Therefore, we can define a change-of-basis matrix $T$ to transform from the basis defined by $\bm{f}_1, \ldots, \bm{f}_C$ onto the standard basis $\mathcal{B}$ for $\Reals^C$. By definition, $T$ must be an orthogonal transformation. Each vector of the standard basis has a single index, unique with respect to all other vectors of the standard basis, that contains a non-zero element.  We can interpret each vector $v \in \mathcal{B} \in \{0,1\}^C$ as encoding an equivalent corresponding set $\hat{\mathcal{V}}$ formed by $\bigcup_{i=1}^{C} \{i \times v_i\} - \{0\}$ (the set of all non-zero indices of that vector). Under this construction, each pair of vectors $u \neq v \in \mathcal{B}$ has equivalent exclusive sets $\hat{\mathcal{U}}$ and $\hat{\mathcal{V}}$, since the indices of their non-zero elements must differ, and therefore the collection must also be exclusive.
\end{proof}

\section{Connections to other related areas}
\label{sec:additionalRelatedWork}
\Revision{Since our approach spans multiple subjects that are usually studied separately (novelty detection and continual learning), we draw connections to a wide variety of prior work.  Some of the connections worth considering, but outside the scope of our \nameref{sec:relatedWork} section, are discussed below.

\textbf{Few-shot Learning.} Similar ideas to the high-level exclusive and low-level shared feature representation of SHELS have been previously explored in the few-shot transfer learning literature.  Networks trained with few-shot learning approaches based on MAML~\citep{finn2017model, antoniou2019train} result mainly in changes at the later layers in a network (i.e., changes in higher level features), creating a mostly immutable shared set of low-level features.  Protonets use a learnt feature embedding over which to make classification decisions, which can be interpreted as a shared low-level feature representation~\citep{snell2017prototypical}.  Few-shot learning approaches have not, however, addressed the notion of high-level \textit{set-based exclusivity} discussed in our work.  Such a concept is less relevant to the few-shot learning paradigm, which is more concerned with learning models that can be quickly adjusted to new tasks when new data is presented.  In contrast, OOD detection approaches must differentiate novel OOD data \textit{before} any model updates, and so this notion of maintaining high-level exclusivity (among both previously seen and OOD classes) becomes very important in the OOD detection paradigm.  Few-shot learning approaches such as MAML address a transfer learning problem where they do not need to maintain performance on previous tasks, and are thus more suited to fast adaptation rather than novelty detection. Further, the prior work in few-shot learning doesn’t explicitly discuss or evaluate the low-level shared or high-level exclusive feature representations that our approach explicitly encourages in the loss function, or that we explicitly evaluate with an exclusivity metric.

\textbf{Modular Architectures.} One tangentially related line of work that has studied sparsity in continual learning is the discovery of modular and compositional architectures. In this case, sparsity is manually embedded into the design of the network via modularity, instead of discovered by the agent as in our approach. The majority of such methods freeze the weights of the modules immediately after training them on a task~\citep{reed2016neural,fernando2017pathnet,valkov2018houdini,li2019learn,veniat2021efficient}, unlike our approach that automatically detects which weights to freeze based on their importance to previous tasks. Other methods instead continually update the weights of the modules without freezing them~\citep{rajasegaran2019random,chen2020mitigating,mendez2021lifelong,qin2021bns,ostapenko2021continual}.}

\section{Cosine normalization induces exclusive feature sets}

\label{sec:CosineNormalizationInducesExclusiveSets}

To empirically verify the ability of cosine normalization to lead to exclusive feature sets, we trained a DNN on five randomly sampled classes from the FMNIST dataset and measured the exclusivity of the high-level feature activations of the classes. Given two sets of feature activations, we measured their exclusivity as the ratio of the total number of dissimilar activations to the total number of activation at the highest-level representation:
\begin{equation}
    \label{equ:activationSimilarity}
    \mathit{exclusivity}(\inputs_{1}, \inputs_{1}) = \frac{\hat{\features}^{L-1}(\inputs_{1}) \oplus \hat{\features}^{L-1}(\inputs_{2})} {\hat{\features}^{L-1}(\inputs_{1}) \lor \hat{\features}^{L-1}(\inputs_{2})}\enspace,
\end{equation}
where $\hat{\features}^{L-1}(\inputs)$ is the unit vector of the final high-level features.  A score of 1.0 indicates perfect exclusivity. We used three methods for training the DNN: a standard linear layer, the method of \citet{techapanurak2020hyperparameter}, and the cosine similarity formulation of Equation~\ref{equ:CosineSimilarity}. Table~\ref{exclusivity_metric} shows the resulting exclusivity for a sample of five classes, while Figure~\ref{fig:avg_exclusivity} shows aggregated results over ten trials, clearly demonstrating that our proposed method is the only one that finds a collection of exclusive feature sets. \Revision{\citet{techapanurak2020hyperparameter} train their model using the softmax of the scaled cosine similarity scores in the loss function, while our approach uses the raw cosine similarity scores. This choice to use the raw scores, although seemingly minor, makes a fundamental difference in learning exclusive feature sets.}

\begin{table}[t]
\caption{Exclusivity metric across 5 randomly sampled classes.}
\label{exclusivity_metric}
\captionsetup{position=top}
    \subfloat[\centering Standard linear layer method]
    {{\scalebox{0.9}{
    \setlength\tabcolsep{1.5pt}
    \begin{tabular}{|c||c|c|c|c|c|}
        \hline
         classes & pullover & coat & bag & dress & boot
          \\\hline \hline
        pullover   &  0  & 0.73  & 0.68  & 0.58  & 0.74
        \\\hline
        coat   & 0.72  & 0  &  0.75  & 0.76  & 0.73
        \\\hline
        bag   & 0.68  & 0.75  & 0  & 0.69  & 0.68
         \\\hline
        dress   & 0.57  & 0.76  & 0.69  & 0  & 0.71
         \\\hline
        boot   & 0.74  & 0.73  & 0.68  & 0.71  & 0
        \\ \hline
        \end{tabular}}}}
        \hfill
      \subfloat[\centering Scaled cosine similarity method (Baseline)]{{
    \scalebox{0.9}{
    \setlength\tabcolsep{1.5pt}
    \begin{tabular}{|c||c|c|c|c|c|}
        \hline
         classes & pullover & coat & bag & dress & boot
          \\\hline \hline
        pullover   &  0  & 0.74  & 0.98  & 0.95  & 1.0
        \\\hline
        coat   & 0.74  & 0  &  0.98  & 0.76  & 1.0
        \\\hline
        bag   & 0.98  & 0.98  & 0  & 1.0  & 1.0
         \\\hline
        dress   & 0.95  & 0.76  & 1.0  & 0  & 1.0
         \\\hline
        boot   & 1.0  & 0.97  & 1.0  & 1.0  & 0
        \\ \hline
        \end{tabular}}}}
        \hfill
     \subfloat[\centering Cosine normalization with raw cosine similarity method (ours)]{{
     \scalebox{0.9}{
     \setlength\tabcolsep{1.5pt}
    \begin{tabular}{|c||c|c|c|c|c|}
        \hline
        classes & pullover & coat & bag & dress & boot
          \\\hline \hline
        pullover   &  0  & 1.0  & 1.0  & 1.0  & 1.0
        \\\hline
        coat   & 1.0  & 0  &  1.0  & 1.0  & 1.0
        \\\hline
        bag   & 1.0  & 1.0  & 0  & 1.0  & 1.0
         \\\hline
        dress   & 1.0  & 1.0  & 1.0  & 0  & 1.0
         \\\hline
        boot   & 1.0  & 1.0  & 1.0  & 1.0  & 0
        \\ \hline
        \end{tabular}}}}
\end{table}

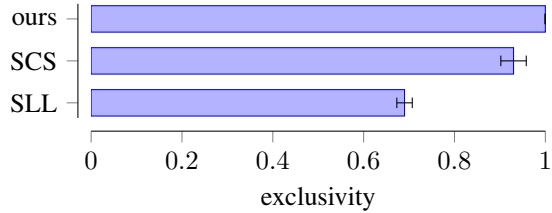
\begin{figure}[t]
\centering
    \begin{tikzpicture}
    \begin{axis} [
        axis x line*=none,
        axis y line*=none,
        axis line shift=5pt,
        xbar, xmin=0, xmax=1.0,
        width=3in,
        height=3.1cm,
        bar width=10pt,
        enlarge y limits=0.18,
        xlabel={exclusivity},
        xtick align=outside,
        symbolic y coords={SLL,SCS,ours},
        ytick=data,
    ]
    \addplot+ [
        error bars/.cd,
            x dir=both,
            x explicit,
            error bar style={color=black},
    ]
    coordinates {
        (0.69,SLL) +- (0.017, 0)
        (0.93,SCS) +- (0.028, 0)
        (1.0,ours) +- (0.001,0)
    };
    \end{axis}
    \end{tikzpicture}
    \caption{Exclusivity averaged across 10 experiments of 5 randomly sampled classes from the FMNIST dataset comparing standard linear layer method, scaled cosine similarity method and our method. SSL represents the standard linear layer method and SCS is the baseline scaled cosine similarity method. Error bars denote standard errors.}
    \label{fig:avg_exclusivity}
\end{figure}

 The architecture and hyperparameters used for the experiment are the same as the ones used for the novelty detection experiments with FMNIST (see Appendix~\ref{sec: implementation_details}) and are kept consistent across the three different methods. We also ensure that the randomly sampled classes for the 10 trials are the same for all three methods to ensure a fair comparison.

 \section{Using Zhou et al.'s exclusive sparsity as an alternative}
 \label{sec:ExclusiveSparsityFeatureSets}

Here, we compare the degree of exclusivity in the feature sets when the model is trained using the exclusive sparsity regularizer  of~\cite{Zhou2010ExclusiveLF} in comparison with our cosine normalization method.  Notably, Zhou et al.'s regularizer has a very different interpretation of exclusivity to what we use in this work:  it encourages {\em individual} features to be exclusive to a class (limiting transfer), where as our approach encourages {\em sets} of features to be exclusive (unique) to a class. We use the exclusive sparsity regularizer as defined by~\cite{Zhou2010ExclusiveLF}:
\begin{equation}
        \label{equ:ExclusiveSparsity}
        \Omega_{E} (\weights^{l}) = \frac{1}{2}\sum_{g} {\| \weights_{g}^{l} \|}_{1}^2 = \frac{1}{2}\sum_{g}\biggl(\sum_{i} \| \weights_{[g,i]}^{l} \| \biggr)^2\enspace.
    \end{equation}
To learn the model we combine the exclusive sparsity regularizer with the standard cross entropy loss:
\begin{equation}
     \label{equ:Loss_exl}
    \mathcal{L}(\weights) = \mathcal{L}_{CE} + \gamma\sum_{l}\Big(\mu^{(l)}\Omega_{E} (\weights^{l})\Big)\enspace.
\end{equation}
The parameter $\gamma$ controls the effect of regularization and is set to $0.001$. We set $\mu^{(l)} = \frac{l-1}{L-1}$ to encourage less exclusivity at the lower layers and more exclusivity at the higher layers, with $\mu^{(1)}=0$ to $\mu^{(L)}=1$ at the outermost layers. We trained the DNN on five randomly sampled classes from the FMNIST dataset and measured the exclusivity of the high-level feature activations using Equation~\ref{equ:activationSimilarity}, similar to the experiment in Appendix~\ref{sec:CosineNormalizationInducesExclusiveSets}. The architecture and hyperparameters are same as  for the novelty detection experiments with FMNIST (see Appendix~\ref{sec: implementation_details}). Figure~\ref{fig:exclsp_avg} shows aggregated results over ten random seeds, comparing the Zhou et al.'s exclusive sparsity method to our method, clearly demonstrating that our method results in more exclusive high-level feature sets.

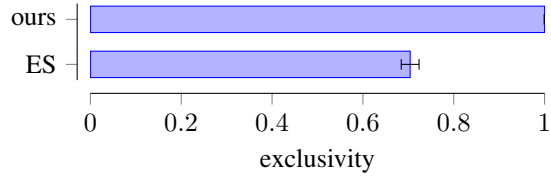
\begin{figure}[h]
\centering
    \begin{tikzpicture}
    \begin{axis} [
        axis x line*=none,
        axis y line*=none,
        axis line shift=5pt,
        xbar, xmin=0, xmax=1.0,
        width=3in,
        height=2.6cm,
        bar width=10pt,
        enlarge y limits=0.35,
        xlabel={exclusivity},
        xtick align=outside,
        symbolic y coords={ES,ours},
        ytick=data,
    ]
    \addplot+ [
        error bars/.cd,
            x dir=both,
            x explicit,
            error bar style={color=black},
    ]
    coordinates {
        (0.704,ES) +- (0.0196, 0)
        (1.0,ours) +- (0.001, 0)
    };
    \end{axis}
    \end{tikzpicture}
    \caption{Exclusivity averaged across 10 experiments of 5 randomly sampled classes from the FMNIST dataset comparing Zhou et al.'s Exclusive Sparsity (ES) method and cosine normalization (ours) method for generating exclusive feature sets. Error bars denote standard errors.}
    \label{fig:exclsp_avg}
\end{figure}

\section{Experimental analysis of the effects of group sparsity}
\label{sec:additionalOODExp}

\Revision{We ran additional novelty detection experiments to analyze the effects of group sparsity on OOD detection by varying both the group sparsity regularization weight $\alpha$ and the layer-specific mixing weight $\mu^{(l)}$ in Equation~\ref{equ:Loss}.  These experiments followed the same procedure and used the same hyperparameters as the within-datasets experiments of Section~\ref{sec:detect-exp}, with our approach divided into two conditions:  (1) SHELS with decreasing group sparsity as the layers increase by setting $\mu^{(l)}=\frac{l-1}{L-1}$ (layer-wise), and (2) SHELS with a constant amount of group sparsity at each layer by setting $\mu^{(l)}=0$ (constant).  Additionally, we varied $\alpha$ on the interval $[0, 0.02]$ in increments of $0.001$ to examine the effects of increasing the contribution of the group sparsity term in the loss function.  The results are shown in Figure~\ref{fig:group-sparsity}.}

\begin{figure}[b!]
    \centering
    \begin{subfigure}{0.5\textwidth}
        \centering
        \includegraphics[height=2in,clip, trim=.05in .0in .15in .15in]{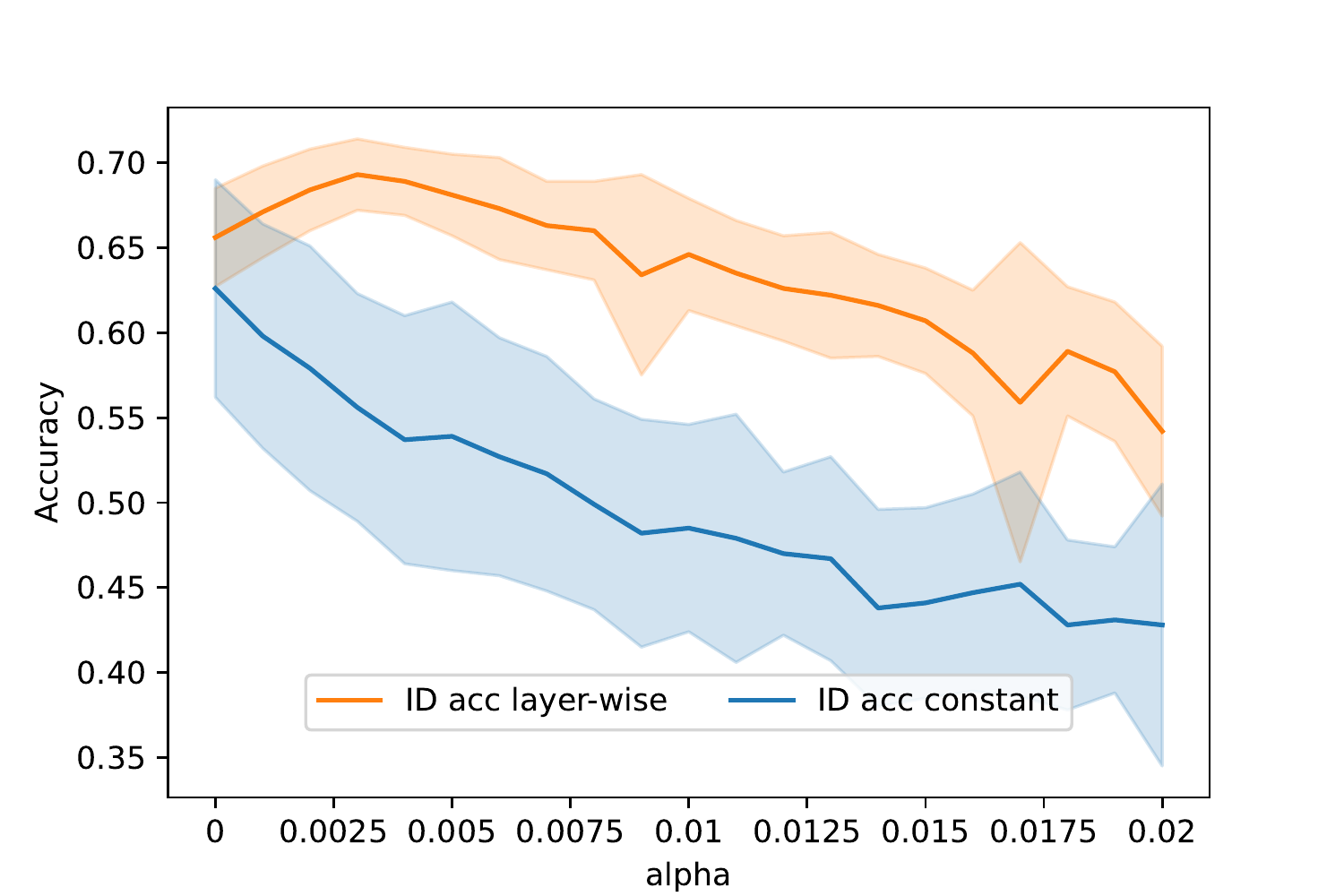}
        \caption{\Revision{ID classification accuracy for CIFAR10}}
         \label{fig:gs-id-cifar}
    \end{subfigure}%
    \begin{subfigure}{0.5\textwidth}
        \centering
        \includegraphics[height=2in,clip, trim=.05in .0in .15in .15in]{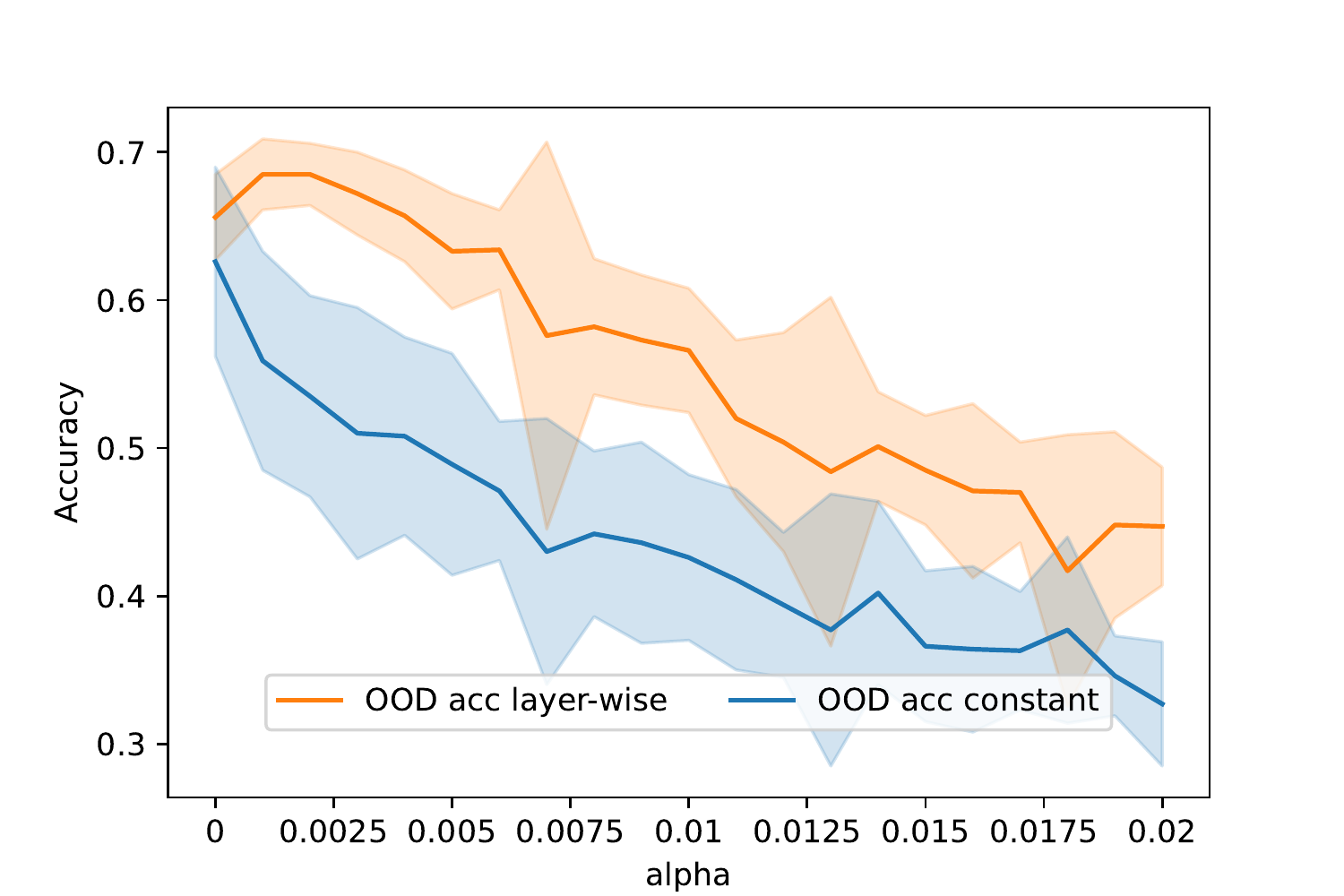}
        \caption{\Revision{OOD detection accuracy for CIFAR10}}
         \label{fig:gs-ood-cifar}
    \end{subfigure}%
    \\
  \begin{subfigure}{0.5\textwidth}
        \centering
        \includegraphics[height=2in,clip, trim=.05in .0in .15in .15in]{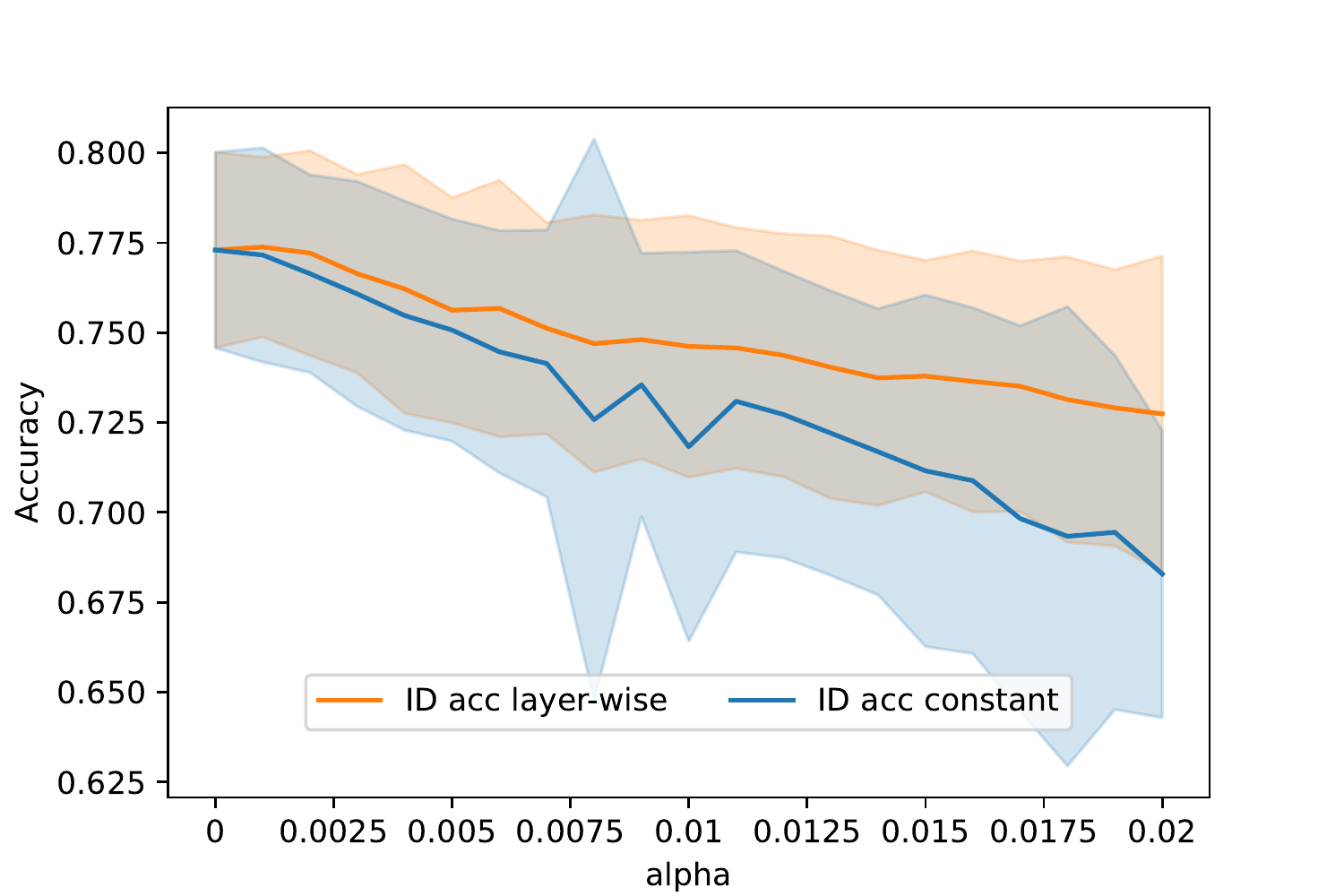}
        \caption{\Revision{ID classification accuracy for FMNIST}}
         \label{fig:gs-id-fmnist}
    \end{subfigure}%
    \begin{subfigure}{0.5\textwidth}
        \centering
        \includegraphics[height=2in,clip, trim=.05in .0in .15in .15in]{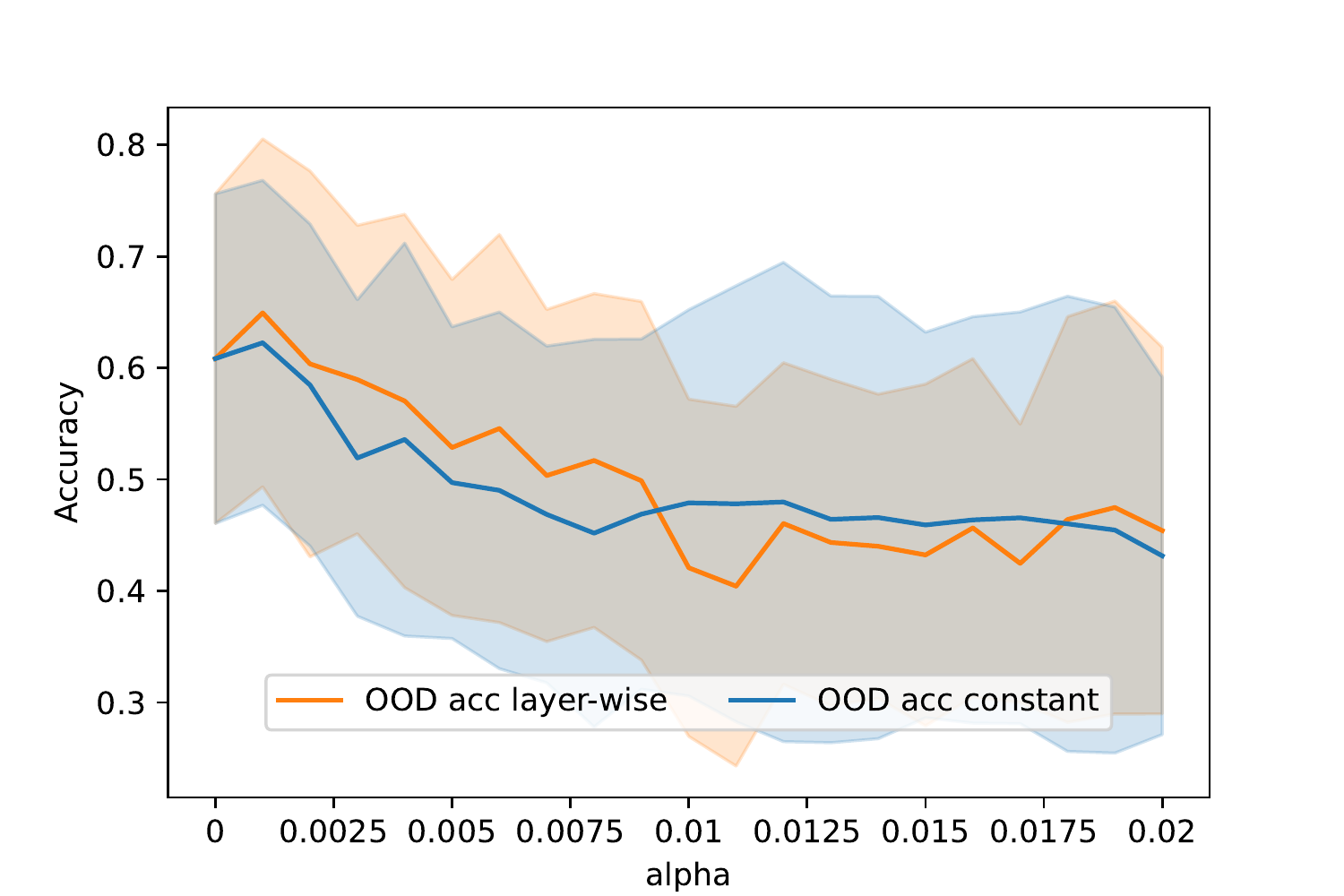}
        \caption{\Revision{OOD detection accuracy for FMNIST}}
         \label{fig:gs-ood-fmnist}
    \end{subfigure}%
    \caption{Analysis of the effects of group sparsity, showing mean $\pm$ standard deviation over ten trials.}
    \label{fig:group-sparsity}
\end{figure}

\Revision{The results show a consistent trend that increasing the weight of group sparsity regularization has a detrimental effect on both ID classification and OOD detection accuracy.  This is an important consideration for combining OOD detection and continual learning, as our continual learning ablation studies presented in Section~\ref{sec:accom-experiment} show that group sparsity regularization is essential to SHELS's effective continual learning performance.  As such, finding a way to incorporate group sparsity regularization into novelty detection models without negatively impacting their performance is a key contribution of our novelty detection approach.  Based on the above experiments, we use the decreasing layer-wise implementation for our method, which is responsible for the low-level shared feature learning of the SHELS representation, and a value of $\alpha=0.001$ for all additional experiments.}

\section{Supplemental illustrations of our approach}
\label{sec:SupplementalIllustrations}

Figure~\ref{fig:catastrophic_forgetting} illustrates forgetting based on model drift and negative transfer, and its interactions with node importance.

\renewcommand{\circlesize}{0.4cm}
\renewcommand{\childspace}{0.25}
\renewcommand{\parentspace}{0.75}
\renewcommand{\diagspace}{1}
\renewcommand{\legendspace}{0.08}
\renewcommand{\legendtext}{\scriptsize}
\definecolor{color1}{rgb}{0.3372549 , 0.70588235, 0.91372549} 
\definecolor{color2}{rgb}{0.9254902 , 0.88235294, 0.2       } 
\definecolor{color3}{rgb}{0.83529412, 0.36862745, 0.        } 
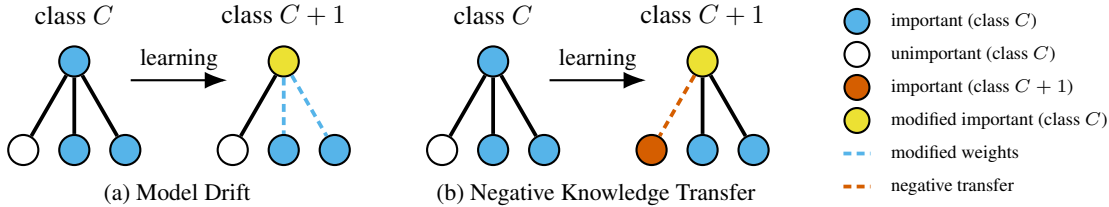
\begin{figure}[h!]
\centering
    \subfloat[Model Drift\label{fig:catastrophic_forgetting:model_drift}] {
        \begin{tikzpicture}[main/.style={draw, thick, circle, minimum size=0.4cm},label distance=0.15cm]
            \node[main] at (0,0) (D1) {};
            \node[main,fill=color1] (D2) [right=\childspace of D1] {};
            \node[main,fill=color1] (D3) [right=\childspace of D2] {};
            \node[main,fill=color1] (S1) [above=\parentspace of D2,label=above:class $C$] {};

            \draw[line width=0.5mm] (S1) -- (D1);
            \draw[line width=0.5mm] (S1) -- (D2);
            \draw[line width=0.5mm] (S1) -- (D3);

            \node[main] (D4) [right=\diagspace of D3] {};
            \node[main,fill=color1] (D5) [right=\childspace of D4] {};
            \node[main,fill=color1] (D6) [right=\childspace of D5] {};
            \node[main,fill=color2] (S2) [above=\parentspace of D5, label=above:class $C+1$] {};

            \draw[line width=0.5mm]                    (S2) -- (D4);
            \draw[line width=0.5mm,densely dashed,color1] (S2) -- (D5);
            \draw[line width=0.5mm,densely dashed,color1] (S2) -- (D6);

             \draw[-{Latex[length=3mm]},thick] ($(S1)+(0.75cm,-0.25cm)$) -- ($(S2)-(0.75cm,0.25cm)$) node [pos=0.5,auto] {\small learning};
        \end{tikzpicture}
    }
    \hfil
    \subfloat[Negative Knowledge Transfer\label{fig:catastrophic_forgetting:negative_knowledge_transfer}] {
        \begin{tikzpicture}[main/.style={draw, thick, circle, minimum size=0.4cm},label distance=0.15cm]
            \node[main] at (0,0) (D1) {};
            \node[main,fill=color1] (D2) [right=\childspace of D1] {};
            \node[main,fill=color1] (D3) [right=\childspace of D2] {};
            \node[main,fill=color1] (S1) [above=\parentspace of D2, label=above:class $C$] {};

            \draw[line width=0.5mm] (S1) -- (D1);
            \draw[line width=0.5mm] (S1) -- (D2);
            \draw[line width=0.5mm] (S1) -- (D3);

            \node[main,fill=color3] (D4) [right=\diagspace of D3] {};
            \node[main,fill=color1] (D5) [right=\childspace of D4] {};
            \node[main,fill=color1] (D6) [right=\childspace of D5] {};
            \node[main,fill=color2] (S2) [above=\parentspace of D5, label=above:class $C+1$] {};

            \draw[line width=0.5mm,densely dashed,color3] (S2) -- (D4);
            \draw[line width=0.5mm]                (S2) -- (D5);
            \draw[line width=0.5mm]                (S2) -- (D6);

            \draw[-{Latex[length=3mm]},thick] ($(S1)+(0.75cm,-0.25cm)$) -- ($(S2)-(0.75cm,0.25cm)$) node [pos=0.5,auto] {\small learning};
        \end{tikzpicture}
    }
    \hfil
    \raisebox{-0.5cm}{
    \begin{tikzpicture}[main/.style={draw, thick, circle},
                        minimum size=0.15cm]
        \node[main,white] at (0,0) (L6a) {};
        \draw[line width=0.5mm,densely dashed,color3] (L6a.west) -- (L6a.east);
        \node (L6b) [right=0.15 of L6a] {\legendtext negative  transfer};
        \node[main,white] (L5a) [above=\legendspace of L6a] {};
        \draw[line width=0.5mm,densely dashed,color1] (L5a.west) -- (L5a.east);
        \node (L5b) [right=0.15 of L5a] {\legendtext modified weights};
        \node[main,fill=color2] (L4a) [above=\legendspace of L5a] {};
        \node (L4b) [right=0.15 of L4a] {\legendtext modified important (class $C$)};
        \node[main,fill=color3] (L3a) [above=\legendspace of L4a] {};
        \node (L3b) [right=0.15 of L3a] {\legendtext important (class $C+1$)};
        \node[main] (L2a) [above=\legendspace of L3a] {};
        \node (L2b) [right=0.15 of L2a] {\legendtext unimportant (class $C$)};
        \node[main,fill=color1] (L1a) [above=\legendspace of L2a] {};
        \node (L1b) [right=0.15 of L1a] {\legendtext important (class $C$)};
    \end{tikzpicture}
    }
    \caption{Catastrophic forgetting caused by (a) Model Drift and (b) Negative Knowledge Transfer.}
\label{fig:catastrophic_forgetting}
\end{figure}

\section{Implementation details}
\label{sec: implementation_details}
\subsection{Architecture details}
\label{sec:arch_details}

The network architecture details for MNIST, FMNIST, SVHN, GTSRB and \Revision{CIFAR10 (within-dataset)} are given in Tables~\ref{tab:conv-arch} and~\ref{tab:linear-arch}. The architecture consists of six convolutional layers followed by linear and cosine layers. The convolutional layers for each dataset are the same as shown in Table~\ref{tab:conv-arch}, while the linear layers differ across datasets and are detailed for each of the datasets in Table~\ref{tab:linear-arch}. For across dataset experiments with CIFAR10 as ID, we use the PyTorch model for VGG16, pretrained on ImageNet, with the last linear layer replaced with cosine normalization.

\begin{table}[h]
    \begin{center}
     \caption{Convolutional(Conv) layers for MNIST, FMNIST, SVHN and GTSRB network architectures; $n_1 = 1$ for MNIST and FMNIST and $n_1 = 3$ for SVHN, GTSRB, and \Revision{CIFAR10 (within-dataset).}}
     \label{tab:conv-arch}
    \begin{tabular}{l|lllll}
        Layer  & Channel(In) & Channel(Out) & Kernel & Stride & Padding
        \\ \hline
        Conv-1  &       $n_1$ & $32$  &   $3 \times 3$ & $1$ & $1$ \\
        Conv-2  &       $32$    & $32$  &  $ 3 \times 3$   & $1$ & $1$ \\
        Maxpool &       $32$    & $32$  &   $2 \times 2$    & $2$
        \\
        Conv-3  &       $32$    & $64$ &    $3 \times 3$    & $1$ & $1$
        \\
        Conv-4  &       $64$ & $64$ &   $3 \times 3$    & $1$ & $1$
        \\
        Maxpool &       $32$ & $32$ &   $2 \times 2$    & $2$
        \\
        Conv-5  &       $64$ & $128$ &  $3 \times 3$   & $1$ & $1$
        \\
        Conv-6  &       $128$ & $128$ & $3 \times 3$  & $1$ & $1$
        \\
        Maxpool &       $32$ & $32$ &   $2 \times 2$    & $2$
        \\
        Flatten()
        \\ \hline

     \end{tabular}
    \end{center}
\end{table}

\begin{table}[h]
    \centering
\captionsetup{position=top}
     \caption{Linear and Cosine layers for MNIST, FMNIST, SVHN, GTSRB and \Revision{CIFAR10 (within-dataset)} network architectures; C is the total number of ID classes for the experiment.}
    \label{tab:linear-arch}

    \subfloat[MNIST and FMNIST]{{
    \begin{tabular}{l|l}
        Layer   & Out-layers
        \\ \hline
        Linear-1 & $256$
        \\
        Cosine  & $C$
        \\ \hline

     \end{tabular}}}
\qquad
    \subfloat[GTSRB and \Revision{CIFAR10 (within-dataset)}]{{
    \begin{tabular}{l|l}
        Layer   & Out-layers
        \\ \hline
        Linear-1 & $2048$
        \\
        Linear-2 & $512$
        \\
        Cosine  & $C$
        \\ \hline

     \end{tabular}}}
     \qquad
    \subfloat[SVHN]{{
    \begin{tabular}{l|l}
        Layer   & Out-layers
        \\ \hline
        Linear-1 & $1024$
        \\
        Linear-2 & $512$
        \\
        Cosine  & $C$
        \\ \hline
     \end{tabular}}}
\end{table}

\subsection{Hyperparamters}

We train models for all the experiments consisting of the datasets MNIST, FMNIST, GTSRB, and SVHN using the Adam optimzer with a learning rate of $0.0001$. For across-datasets experiments with the CIFAR10 dataset, we train the models using the SGD optimizer with an initial learning rate of $0.001$, decaying it by a factor of $2$ after every four epochs, momentum of $0.9$, and weight decay of $5\mathrm{e}{-4}$. \Revision{For within-dataset experiments with CIFAR10 we use the Adam optimzer with a learning rate of $0.0001$}. For the novelty detection experiments in Section~\ref{sec:detect-exp}, for all datasets, we use a batch size of $32$ and set $\alpha$ to $0.001$. The total number of training epochs varies for the different datasets: for MNIST and GTSRB we use $10$ epochs, for FMNIST and SVHN we use $20$ epochs, and for CIFAR10 we use $30$ epochs. For the novelty accommodation experiment and the combined novelty detection and accommodation experiment in Sections~\ref{sec:accom-experiment} and~\ref{sec:detect-accomm-exp} with the MNIST dataset, we use the same hyperparameters from the novelty detection experiment for MNIST. Additionally, we set $\beta$ to $2000$, decrease the learning rate by a factor of $10$, and use an OOD confidence score threshold of 0.75. For the novelty accommodation experiment and the combined novelty detection and accommodation experiment using FMNIST in Appendix~\ref{sec:additionaleClExp} and Section \Revision{\ref{sec:detect-accomm-exp}}, we use the same hyperparameters from the novelty detection experiment for FMNIST in addition to setting $\beta$ to $3000$, decreasing the learning rate by a factor of $2$, and using an OOD confidence score threshold of 0.3. \Revision{For the novelty accommodation experiments using GTSRB in Appendix~\ref{sec:additionaleClExp}, we use the same hyperparameters from the novelty detection experiment for GTSRB in addition to setting $\beta$ to $50000$, and keep the same learning rate.

To ensure a fair comparison with the AGS-CL baseline, for the AGS-CL experiments we set $\mu$ to be the same as $\alpha$ from the corresponding SHELS experiment to ensure the same amount of sparsity regularization. We also set $\lambda$ to be the same as $\beta$ from the corresponding SHELS experiment to ensure the same amount of weight freezing. However, we did find it necessary to tune the learning rate for AGS-CL: for MNIST we reduced the learning rate by a factor of $10$, for FMNIST we reduced the learning rate by a factor of $8$, and for GTSRB we reduced the learning rate by a factor of $2$. All AGS-CL models were trained using the Adam optimizer with an initial learning rate of $0.0001$ and a fixed batch size of 32.
For GDUMB experiments we use $K = 1\%$ of the training data, SGD optimizer, fixed batch size of 32, learning rates [0.05,0.0005] and SGDR schedule with $\mathrm{T0} = 1$, $\mathrm{Tmult} = 2$. For SHELS-GDUMB experiments we also set  $K = 1\%$ of the training data and use the Adam optimizer with a learning rate of $0.0001$ and a fixed batch size of 32. For both GDUMB and SHELS-GDUMB experiments, we use all the ID training data to warm start the model and use the replay buffer when we begin to incrementally learn OOD classes.}
\Revision{For the MAHA baseline, we set the threshold scores for each ID class to one standard deviation minus the mean of the scores of all the correctly identified samples from the training data of the corresponding class. Also, we set the same OOD confidence score threshold for MAHA as that used by the corresponding SHELS experiment.}

\subsection{Dataset details}

We detail the dataset splits for novelty detection experiments from Section~\ref{sec:detect-exp} in Tables~\ref{tab:datasplit1} and~\ref{tab:datasplit2}. For the within-dataset experiments, the OOD set comprises of a subset of default test set of that dataset corresponding to the OOD classes, while for the across-datasets experiments, the OOD set consists of the entire default test set of the OOD dataset.

\begin{table}[b]
    \centering
    \caption{Dataset sizes for novelty detection within-dataset experiments.}
    \label{tab:datasplit1}
    \begin{tabular}{c|c|c|c|c|c}
          & MNIST & FMNIST & SVHN & GTSRB & CIFAR10
          \\\hline
        train set & $\sim26000$    & $26400$ & $\sim31000$ & $\sim19000$ & $25000$ \\
        \hline
        val set   & $\sim3500$     & $3600$  & $\sim4000$  & $\sim2600$ & $3000$ \\
        \hline
        test set  & $\sim4800$     & $5000$  & $\sim12000$ & $\sim6900$ & $5000$ \\
        \hline
        OOD set   &$\sim5000$      & $5000$  & $\sim13000$ & $\sim5500$ & $5000$ \\
        \hline
    \end{tabular}
\end{table}

\begin{table}[t]
     \begin{minipage}[t]{0.58\linewidth}
    \captionsetup{width=0.95\linewidth}
     \caption{Dataset details for novelty detection across-datasets experiment.}
    \begin{tabular}{c|c|c|c|c|c}
          & MNIST & FMNIST & SVHN & GTSRB & CIFAR10
          \\\hline
        train set &$52800$ &$52800$ &$64466$  &$34502$  & $50000$ \\
        \hline
        val set &$7200$ &$7200$ & $8791$  & $4707$ & $6000$\\
        \hline
        test set &$10000$ &$10000$ & $26032$ & $12569$ & $10000$\\
        \hline
    \end{tabular}
    \label{tab:datasplit2}
    \end{minipage}
    \hfill
    \begin{minipage}[t]{0.38\linewidth}
    \centering
    \captionsetup{width=0.95\linewidth}
    \caption{Input image size in pixels for each dataset.}
    \begin{tabular}{|c|c|}
        \hline
        Dataset & Input Size
        \\ \hline
         MNIST &    $32 \times 32$ \\ \hline
         FMNIST &   $32 \times 32$  \\\hline
         SVHN &     $32 \times 32$  \\ \hline
         GTSRB &    $112 \times 112$ \\ \hline
         \Revision{CIFAR10 (within-dataset)} &  \Revision{$112 \times 112$}  \\\hline
         CIFAR10 (across-datasets) &  $224 \times 224$ \\ \hline
    \end{tabular}
    \label{tab:input_size}
    \end{minipage}
\end{table}

The input image sizes for each dataset compatible with the architectures described in Appendix~\ref{sec:arch_details} are given in Table~\ref{tab:input_size}. For the across-datasets novelty detection experiments, we resize the OOD data image size to the ID data image size.  For example, for the SVHN(ID) vs. CIFAR10(OOD) experiment, we resize the CIFAR10 images to be of the same size as SVHN, i.e., $32 \times 32$ to be compatible with the network architecture.

\section{Statistical testing details}
\label{sec:stat-test-details}

For all across-datasets novelty detection experiments, we performed unpaired Student's t-tests to determine whether the OOD detection method had a statistically significant effect on OOD detection accuracy, ID classification accuracy, combined accuracy, and OOD detection AUROC. This was evaluated on both our SHELS approach and the state-of-the-art OOD detection baseline \citep{techapanurak2020hyperparameter}.  We performed two-tailed tests to test for effects in either direction, making no prior assumption that either method would outperform the other.  We applied Holm-Bonferroni adjustments to protect against type I error from multiple comparisons, since we are testing four measures in each experiment, and use the adjusted p-values to test for statistical significance at a significance level of $\alpha=0.05$.  We use unpaired samples in these experiments, since each trial for each condition was trained independently with a different random seed.  Summary and test statistics for all across-datasets experiments are reported below in Table~\ref{tab:across-datasets-stats}.

For all within-dataset novelty detection experiments, we performed paired Student's t-tests to determine whether OOD detection method had a statistically significant effect on OOD detection accuracy, ID classification accuracy, combined accuracy, and OOD detection AUROC, evaluated between our SHELS approach and the state-of-the-art OOD detection baseline.  We again performed two-tailed tests to test for effects in either direction, making no prior assumption that either method would outperform the other. We applied Holm-Bonferroni adjustments to protect against type I error from multiple comparisons since we are again testing four measures per experiment, and use the adjusted p-values to test for statistical significance at a significance level of $\alpha=0.05$.  Unlike in the across-datasets experiments, we performed this statistical analysis using paired samples, by pairing samples with the same permutations of ID vs.~OOD classes.  For example, trial 1 for the MNIST dataset produced paired samples across both approaches which each consisted of ID classes $\{1,0,8,6,3\}$, trial 2 consisted of ID classes $\{5,7,8,3,9\}$ for both approaches, etc.  Summary and test statistics for all within dataset experiments are reported below in Table~\ref{tab:within-dataset-stats}.

\begin{table}[p]
\small
\caption{Summary statistics (mean $\pm$ std.~dev.) and Student's t-test results for across-datasets novelty detection.}
\label{tab:across-datasets-stats}
\centering
\begin{tabular}{c|ccccc}
Dataset & Measure                     & Baseline & SHELS & p-value & Adjustment \\ \hline
FMNIST & OOD detection acc           & $88.94\pm2.68$ & $87.45\pm5.48$ & $t(18)=0.77, \: p=.45$ & $\tilde{p}=.45$\\
vs. & ID classification acc       & $81.04\pm0.98$ & $76.75\pm0.44$ & $t(18)=12.62, \: p<.0001$ & $\tilde{p}<.001$\\
MNIST & Combined acc & $84.99\pm1.31$ & $82.12\pm2.7$ & $t(18)=3.02, \: p=.007$ & $\tilde{p}=.021$\\
& AUROC & $0.88\pm0.063$ & $0.77\pm0.064$ & $t(18)=3.87, \: p=.0011$ & $\tilde{p}=.0022$ \\ \hline
MNIST & OOD detection acc           & $96.51\pm2.22$ & $99.98\pm0.03$ & $t(18)=4.94, \: p=.0001$ & $\tilde{p}<.001$\\
vs. & ID classification acc       & $93.06\pm1.29$ & $84.65\pm0.42$ & $t(18)=19.6, \: p<.0001$ & $\tilde{p}<.001$\\
FMNIST & Combined acc & $94.78\pm1.15$ & $92.37\pm0.27$ & $t(18)=6.45, \: p<.0001$ & $\tilde{p}<.001$ \\
& AUROC & $0.99\pm0.002$ & $0.94\pm0.03$ & $t(18)=5.26, \: p<.0001$ & $\tilde{p}<.001$ \\ \hline
SVHN & OOD detection acc           & $78.49\pm4.14$ & $90.45\pm1.13$ & $t(18)=8.81, \: p<.0001$ & $\tilde{p}<.001$\\
vs. & ID classification acc       & $84.14\pm0.57$ & $77.11\pm0.67$ & $t(18)=25.27, \: p<.0001$ & $\tilde{p}<.001$\\
CIFAR10 & Combined acc & $81.31\pm2.03$ & $83.78\pm0.55$ & $t(18)=3.71, \: p=.002$ & $\tilde{p}=.004$\\
& AUROC & $0.90\pm0.007$ & $0.90\pm0.01$ & $t(18)=0, \: p=1$ & $\tilde{p}=1$ \\ \hline
SVNH & OOD detection acc           & $73.01\pm3.70$ & $84.88\pm1.49$ & $t(18)=9.41, \: p<.0001$ & $\tilde{p}<.001$\\
vs. & ID classification acc       & $84.14\pm0.57$ & $77.10\pm0.67$ & $t(18)=25.31, \: p<.0001$ & $\tilde{p}<.001$\\
GTSRB & Combined acc & $78.57\pm1.92$ & $80.99\pm0.91$ & $t(18)=3.6, \: p=.002$ & $\tilde{p}=.004$\\
& AUROC & $0.88\pm0.009$ & $0.89\pm0.008$ & $t(18)=2.63, \: p=.017$ & $\tilde{p}=.017$ \\ \hline
GTSRB & OOD detection acc           & $97.10\pm1.00$ & $97.08\pm1.08$ & $t(18)=0.04, \: p=.966$ & $\tilde{p}=1$\\
vs. & ID classification acc       & $86.01\pm1.36$ & $84.45\pm0.50$ & $t(18)=3.4, \: p=.003$ & $\tilde{p}=.009$\\
CIFAR10 & Combined acc & $91.56\pm0.96$ & $90.76\pm0.70$ & $t(18)=2.13, \: p=.047$ & $\tilde{p}=.094$\\
& AUROC & $0.96\pm0.01$ & $0.90\pm0.032$ & $t(18)=5.66, \: p<.0001$ & $\tilde{p}<.001$ \\ \hline
GTSRB & OOD detection acc           & $95.94\pm1.71$ & $97.07\pm2.77$ & $t(18)=1.1, \: p=.287$ & $\tilde{p}=.574$\\
vs. & ID classification acc       & $86.01\pm1.36$ & $84.45\pm0.50$ & $t(18)=3.4, \: p=.003$ & $\tilde{p}=.009$\\
SVHN & Combined acc & $90.97\pm1.40$ & $91.21\pm0.45$ & $t(18)=0.52, \: p=.612$ & $\tilde{p}=1$\\
& AUROC & $0.98\pm0.005$ & $0.94\pm0.018$ & $t(18)=6.77, \: p<.0001$ & $\tilde{p}<.001$ \\ \hline
CIFAR10 & OOD detection acc           & $94.84\pm1.97$ & $96.68\pm0.54$ & $t(18)=2.85, \: p=.011$ & $\tilde{p}=.044$\\
vs. & ID classification acc       & $75.53\pm0.46$ & $75.76\pm0.10$ & $t(18)=1.55, \: p=.14$ & $\tilde{p}=.28$\\
SVHN & Combined acc & $85.24\pm1.10$ & $86.21\pm0.25$ & $t(18)=2.72, \: p=.014$ & $\tilde{p}=.042$\\
& AUROC & $0.97\pm0.003$ & $0.95\pm0.05$ & $t(18)=1.26, \: p=.223$ & $\tilde{p}=.223$ \\ \hline
CIFAR & OOD detection acc           & $96.23\pm4.00$ & $97.34\pm0.43$ & $t(18)=0.87, \: p=.394$ & $\tilde{p}=1$\\
vs. & ID classification acc       & $75.68\pm0.46$ & $75.76\pm0.10$ & $t(18)=0.54, \: p=.598$ & $\tilde{p}=1$\\
GTSRB & Combined acc & $85.95\pm2.03$ & $86.45\pm0.36$ & $t(18)=0.77, \: p=.453$ & $\tilde{p}=1$\\
& AUROC & $0.93\pm0.002$ & $0.92\pm0.025$ & $t(18)=1.26, \: p=.223$ & $\tilde{p}=.892$ \\ \hline
\end{tabular}
\end{table}

\begin{table}[p]
\small
\caption{Summary statistics (mean $\pm$ std.~dev.) and Student's t-test results for within-dataset novelty detection.}
\label{tab:within-dataset-stats}
\centering
\begin{tabular}{c|ccccc}
Dataset                & Measure                     & Baseline & SHELS & p-value & Adjustment \\ \hline
\multirow{3}{*}{MNIST} & OOD detection acc           & $77.65\pm8.31$ & $97.95\pm1.26$ & $t(9)=8.14, \: p<.0001$ & $\tilde{p}<.001$\\
                       & ID classification acc       & $92.58\pm0.84$ & $85.07\pm0.69$ & $t(9)=27.21, \: p<.0001$ & $\tilde{p}<.001$ \\
                       & Combined acc & $85.11\pm4.04$ & $91.51\pm0.65$ & $t(9)=5.90, \: p=.0005$ & $\tilde{p}=.001$ \\
                       & AUROC & $0.91\pm0.06$ & $0.96\pm0.02$ & $t(9)=2.80, \: p=.0207$ & $\tilde{p}=.0207$ \\ \hline
\multirow{3}{*}{FMNIST} & OOD detection acc           & $43.64\pm12.76$ & $64.81\pm14.91$ & $t(9)=5.06, \: p=.0007$ & $\tilde{p}=.002$ \\
                       & ID classification acc       & $84.35\pm4.41$ & $77.68\pm2.70$ & $t(9)=11.60, \: p<.0001$ & $\tilde{p}<.001$ \\
                       & Combined acc & $83.99\pm5.75$ & $71.25\pm6.88$ & $t(9)=3.48, \: p=.0069$ & $\tilde{p}=.014$ \\
                       & AUROC & $0.74\pm0.09$ & $0.79\pm0.07$ & $t(9)=2.04, \: p=.072$ & $\tilde{p}=.072$ \\ \hline
\multirow{3}{*}{SVHN} & OOD detection acc           & $51.92\pm7.10$ & $69.24\pm4.75$ & $t(9)=15.40, \: p<.0001$ & $\tilde{p}<.001$ \\
                       & ID classification acc       & $87.49\pm2.06$ & $79.23\pm1.23$ & $t(9)=14.75, \: p<.0001$ & $\tilde{p}<.001$ \\
                       & Combined acc & $69.70\pm2.78$ & $74.23\pm2.27$ & $t(9)=7.13, \: p<.0001$ & $\tilde{p}<.001$ \\
                       & AUROC & $0.81\pm0.02$ & $0.82\pm0.03$ & $t(9)=1.31, \: p=.224$ & $\tilde{p}=.224$ \\ \hline
\multirow{3}{*}{\Revision{CIFAR10}} & \Revision{OOD detection acc}           & \Revision{$52.62\pm5.41$} & \Revision{$59.78\pm6.94$} & \Revision{$t(9)=7.95, \: p<.0001$} & \Revision{$\tilde{p}<.001$} \\
                       & \Revision{ID classification acc}       & \Revision{$67.83\pm3.20$} & \Revision{$67.12\pm2.85$} & \Revision{$t(9)=3.98, \: p=.0032$} & \Revision{$\tilde{p}=.006$} \\
                       & \Revision{Combined acc} & \Revision{$60.22\pm3.85$} & \Revision{$63.45\pm4.40$} & \Revision{$t(9)=7.23, \: p=.0001$} & \Revision{$\tilde{p}<.001$} \\
                       & \Revision{AUROC} & \Revision{$0.69\pm0.07$} & \Revision{$0.70\pm0.06$} & \Revision{$t(9)=0.33, \: p=.3305$} & \Revision{$\tilde{p}=.331$ }\\ \hline
\multirow{3}{*}{GTSRB} & OOD detection acc           & $76.35\pm13.06$ & $81.43\pm4.57$ & $t(9)=1.21, \: p=.257$ & $\tilde{p}=.514$ \\
                       & ID classification acc       & $88.71\pm2.01$ & $83.94\pm1.26$ & $t(9)=6.32, \: p=.0001$ & $\tilde{p}<.001$ \\
                       & Combined acc & $82.53\pm7.28$ & $82.68\pm2.24$ & $t(9)=0.07, \: p=.949$ & $\tilde{p}=1$ \\
                       & AUROC & $0.82\pm0.06$ & $0.88\pm0.06$ & $t(9)=3.36, \: p=.008$ & $\tilde{p}=.024$ \\ \hline
\end{tabular}
\end{table}

\section{Additional continual learning experiments}
\label{sec:additionaleClExp}

We ran additional novelty accommodation experiments on the FMNIST \Revision{and GTSRB} datasets.
We use four randomly sampled classes from FMNIST as ID and learn three out of the remaining six OOD classes. We learn a subset of the OOD classes for these experiments due to network capacity limitations of the architectures described in Section~\ref{sec:arch_details}.  In order to accommodate the full set of FMNIST classes, one could dynamically add nodes to the SHELS model, which we discuss briefly in Section~\ref{sec:conclusion} as future work. Figure~\ref{cl_results_fmnist} details the class-incremental learning performance with known class boundaries, following the approach described in Section~\ref{sec:novelty-accomm}. We include the same \Revision{baselines and} ablation studies. \Revision{Similar to the trends seen with MNIST, SHELS outperforms AGS-CL. The AGS-CL baseline performs similar to SHELS without group sparsity, further reinforcing that SHELS is more effective at inducing sparsity enabling the model to learn new features for future classes. As expected, GDUMB still serves as an upper performance bound for SHELS due to its use of replay data.}
The ablations show significant catastrophic forgetting without weight penalization (blue curves), and an inability to accommodate novel data due to the lack of network capacity without group sparsity (green curves).

\Revision{For GSTRB we use 23 randomly sampled classes as ID and the remaining 20 classes as OOD. We learn nine out of the 20 classes incrementally due to network capacity limitations as described before for the FMNIST experiments. We have also found that our model performs best when trained on relatively diverse in-distribution data. As such, our continual learning experiments train on an ID dataset containing multiple classes. In cases where this is not feasible, the model can be pre-trained with other available datasets and deployed for continual learning.

We demonstrate the performance of SHELS in GTSRB following the approach described in Section~\ref{sec:novelty-accomm} and also include the baselines AGS-CL and GDUMB for comparison as shown in Figure~\ref{fig:cl_results_gtsrb}. We observe similar trends as seen with MNIST and FMNIST: SHELS outperforms AGS-CL while GDUMB serves as an upper performance bound.

To create a more fair comparison between SHELS, which does not retain previous class data, and replay-based methods, such as GDUMB, we combined the two approaches into a new SHELS-GDUMB algorithm. This new method incorporates a size $K$ replay buffer, as described in the original GDUMB paper~\citep{prabhu2020gdumb}, into SHELS. Figure~\ref{fig:cl_results_gtsrb} shows that SHELS-GDUMB outperforms SHELS, since it has access to previous ID data while learning the new class, making it significantly easier to avoid catastrophic forgetting. This result demonstrates the SHELS representation's compatibility with replay. We observe that GDUMB still performs slightly better than SHELS-GDUMB, due to the weight penalty feature of SHELS, which penalizes updates to previously important nodes, consequently leading to network saturation.  However, note that the GDUMB algorithm on its own (i.e., without SHELS) cannot perform OOD detection, and is thus unsuitable for the combined novelty detection and accommodation paradigm.}


\begin{figure}[t]
    \centering
    \begin{minipage}{0.48\textwidth}
        \centering
        \includegraphics[height=2in,clip, trim=.2in .2in 0.4in .2in ]{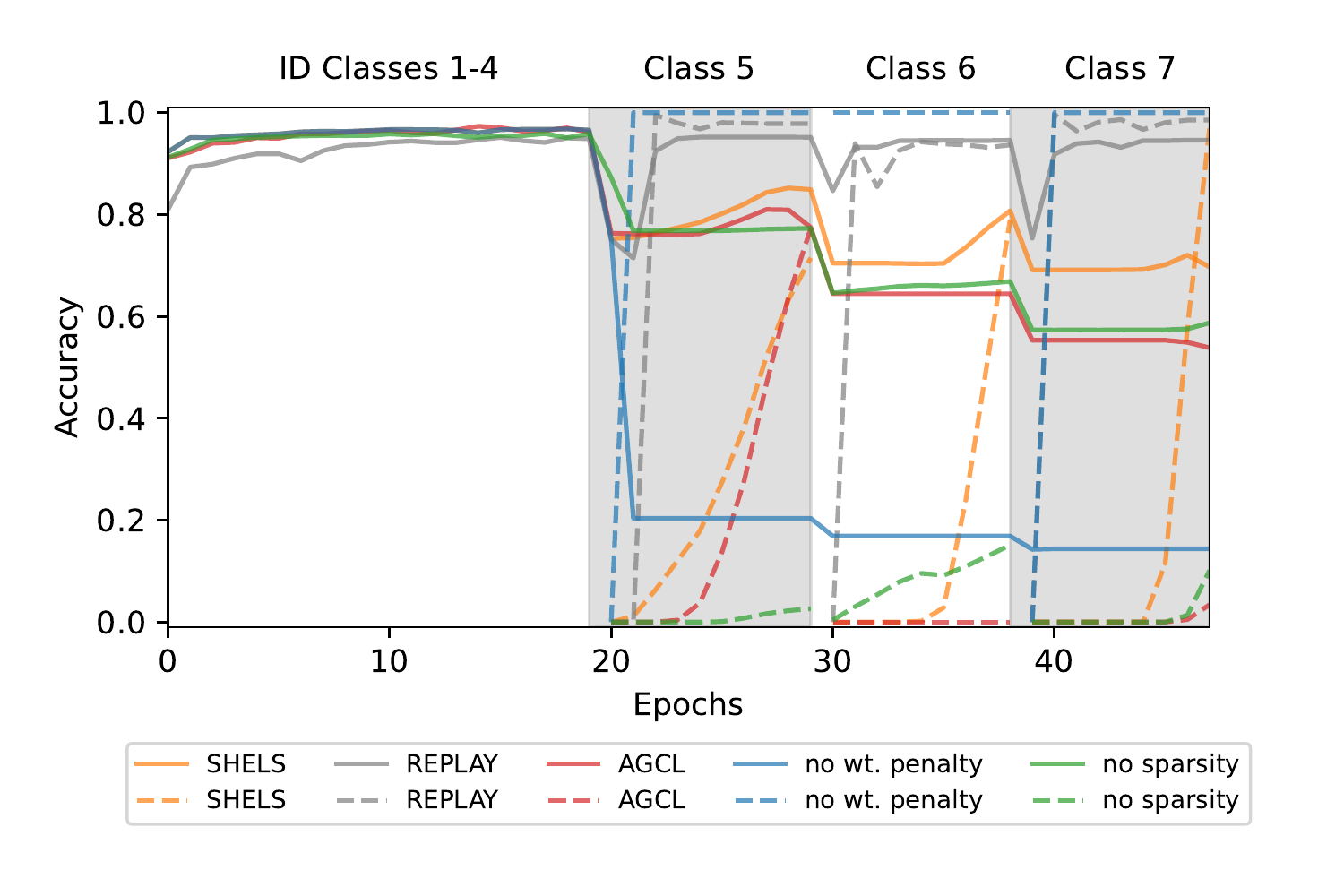}
        \caption{\label{cl_results_fmnist}Class incremental learning on FMNIST, pretrained on four classes, with three more classes learnt incrementally. Solid lines show overall test accuracy, dashes show test accuracy of the novel class only.}
    \end{minipage}
    \hfill
     \begin{minipage}{0.48\textwidth}
        \centering
        \includegraphics[height=2in,clip, trim=.15in .00in .15in .35in]{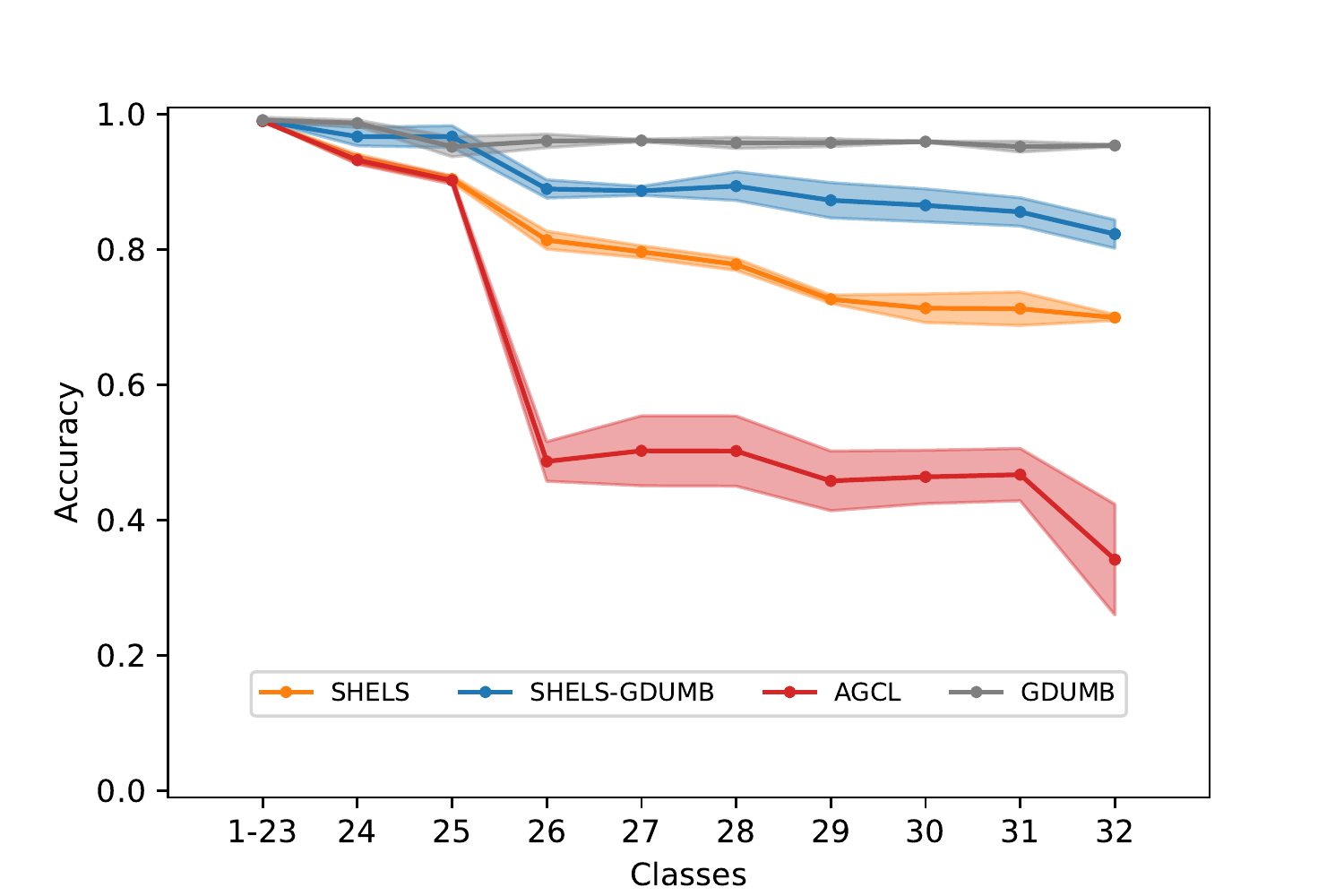}
        \caption{ \label{fig:cl_results_gtsrb}Class incremental learning on GTSRB, pretrained on 23 classes with nine classes learnt incrementally. Curves show overall test accuracy, showing mean $\pm$ standard deviation across three random seeds.}
    \end{minipage}
\end{figure}

\end{document}

%% file: math_commands.tex
\usepackage{amsmath,amsfonts,bm}

\def\eqref#1{equation~\ref{#1}}

\def\1{\bm{1}}

\DeclareMathAlphabet{\mathsfit}{\encodingdefault}{\sfdefault}{m}{sl}
\SetMathAlphabet{\mathsfit}{bold}{\encodingdefault}{\sfdefault}{bx}{n}

\newcommand{\softmax}{\mathrm{softmax}}

\DeclareMathOperator*{\argmax}{arg\,max}